\documentclass[paper=letter]{scrartcl}
\usepackage[latin1]{inputenc}		
\usepackage[T1]{fontenc} 

\title{
  Parallel Bayesian Global Optimization of Expensive Functions\thanks{Peter Frazier and Jialei Wang were partially supported by NSF CAREER CMMI-1254298, NSF CMMI-1536895, NSF IIS-1247696, AFOSR FA9550-12-1-0200, AFOSR FA9550-15-1-0038, and AFOSR FA9550-16-1-0046.}
}

\date{}

\usepackage{authblk}

\author[1]{Jialei Wang\thanks{jw865@cornell.edu}}
\author[2]{Scott C. Clark\thanks{scott@sigopt.com}}
\author[3]{Eric Liu\thanks{eliu@yelp.com}}
\author[1]{Peter I.\ Frazier\thanks{pf98@cornell.edu}}
\affil[1]{School of Operations Research and Information Engineering, Cornell University}
\affil[2]{SigOpt, 244 Kearny St, San Francisco, CA}
\affil[3]{Yelp, Inc., 140 New Montgomery, San Francisco, CA}

\usepackage{natbib}
 \bibpunct[, ]{(}{)}{,}{a}{}{,}%
 \def\bibfont{\small}%
 \def\bibsep{\smallskipamount}%
 \def\bibhang{24pt}%
\newcommand{\argmax}{\operatornamewithlimits{argmax}}
\newcommand{\argmin}{\operatornamewithlimits{argmin}}
\usepackage{amsmath,amssymb,mathtools,amsthm}
\usepackage{algorithmic}
\usepackage[Algorithm,ruled]{algorithm}
\usepackage{bm}
\usepackage{color}
\usepackage{bbm}
\usepackage[margin=1in]{geometry}
\usepackage{setspace}
\usepackage{authblk}
\usepackage{csquotes}
\usepackage{hyperref}
\usepackage{float}
\hypersetup{breaklinks=true}
\usepackage[caption=false]{subfig}
\newcommand{\bx}{\bm{x}}

\newcommand{\x}[1]{\bm{x}^{(#1)}}
\newcommand{\bmu}[1]{\bm{\mu}^{(#1)}}
\newcommand{\y}[1]{y^{(#1)}}
\newcommand{\Sig}[1]{\bm{\Sigma}^{(#1)}} 

\newcommand{\R}{\mathbb{R}}
\newcommand{\E}{\mathbb{E}}
\newcommand{\qEI}{\text{\textit q-EI}}

\newtheorem{theorem}{Theorem}

\newtheorem{corollary}{Corollary}
\newtheorem{lemma}{Lemma}

\begin{document}

\maketitle								

\section*{Abstract}
We consider parallel global optimization of derivative-free expensive-to-evaluate functions, and
propose an efficient method based on stochastic approximation for implementing a conceptual Bayesian
optimization algorithm proposed by \cite{GiLeCa08}. At the heart of this algorithm is maximizing 
the information criterion called the ``multi-points expected improvement'', or the $\qEI$.
To accomplish this, we use infinitessimal
perturbation analysis (IPA) to construct a stochastic gradient estimator and show that this estimator
is unbiased. We also show that the stochastic gradient ascent algorithm using the constructed gradient
estimator converges to a stationary point of the $\qEI$ surface, and therefore, as the number of 
multiple starts of the gradient ascent algorithm and the number of steps for each start grow large, 
the one-step Bayes optimal set of points is recovered.
We show in numerical experiments that our method for maximizing the $\qEI$ is faster 
than methods based on closed-form evaluation using high-dimensional integration, when considering many parallel function evaluations, and is comparable in speed when considering few.
We also show that the resulting one-step Bayes optimal algorithm for parallel global optimization finds high-quality solutions with fewer evaluations than a heuristic based on approximately maximizing the $\qEI$.
A high-quality open source implementation of this algorithm is available in the open source Metrics Optimization Engine (MOE). 

\section{Introduction}
\label{sec:intro}
We consider derivative-free global optimization of expensive functions, in which 
(1) our objective function is time-consuming to evaluate, limiting the number of 
function evaluations we can perform; (2) evaluating the objective function provides 
only the value of the objective, and not the gradient or Hessian; and (3) we seek a global, 
rather than a local, optimum.  Such problems arise when the objective 
function is evaluated by running a complex computer code (see, e.g., \citealt{sacks1989design}), 
performing a laboratory experiment, 
or building a prototype system to be evaluated in the real world. 
In this paper we assume our function evaluations are deterministic, i.e., free from noise.

Bayesian Global Optimization (BGO) methods are one class of methods for solving 
such problems. They were initially proposed by~\cite{kushner1964new}, with 
early work pursued in~\cite{MoTiZi78} and~\cite{mockus1989bayesian}, and more recent work including 
improved algorithms~\citep{boender1987bayesian, JoScWe98, HuAlNoZe06}, convergence analysis~\citep{Ca97,CaZi02, VaBe10}, and allowing noisy function evaluations~\citep{CaZi05,ViVaWa08,frazier2009knowledge, HuAlNoZe06}. 

The most well-known BGO method is Efficient Global Optimization (EGO)
from~\cite{JoScWe98}, which chooses each point at which to evaluate the expensive
objective function in the ``outer'' expensive global optimization problem by
solving an ``inner'' optimization problem: maximize the ``expected improvement''.
Expected improvement is the value of information \citep{Ho66} from a
single function evaluation, and quantifies the benefit that this evaluation provides in terms of revealing a point with a better objective function value than previously known.
If this is the final point that will be evaluated in the outer optimization
problem, and if additional conditions are satisfied (the evaluations are free
from noise, and the implementation decision, i.e., the solution that will be
implemented in practice after the optimization is complete, is restricted to be
a previously evaluated point), then 
the point with 
largest expected improvement is the Bayes-optimal point to evaluate, in the sense of
providing the best possible average-case performance in the outer expensive
global optimization problem \citep{FrazierWang2016}.

Solving EGO's inner optimization problem is facilitated by an
easy-to-compute and differentiate expression for the expected
improvement in terms of the scalar normal cumulative distribution function.
Fast evaluation of the expected improvement and its gradient make it possible
in many applications to solve the inner optimization problem in significantly less time than
the time required per evaluation of the expensive outer objective,
which is critical to EGO's usefulness as an optimization algorithm.

The inner optimization problem at the heart of EGO and its objective, the
expected improvement, was generalized by~\cite{GiLeCa08} to the parallel
setting, in which the expensive objective can be evaluated at several 
points simultaneously.  This generalization, called the ``multi-points expected improvement'' or the $\qEI$,
is consistent with the decision-theoretic derivation of expected improvement and 
quantifies the expected utility that will result from the evaluation of a {\textit set} of points.
(Here, the increase in utility is the improvement in the objective value.)
This work also provided an analytical formula for $q=2$.

If this generalized inner optimization problem, 
which is to find the set of points to evaluate next that jointly maximize the $\qEI$,
could be solved efficiently, then this would provide the one-step Bayes-optimal set 
of points to evaluate in the outer problem, 
and would create a one-step Bayes-optimal algorithm for global optimization of expensive functions 
able to fully utilize parallelism.

This generalized inner optimization problem is challenging, however, because unlike the scalar expected improvement used by EGO,
the $\qEI$ lacks an easy-to-compute and differentiate expression, and is calculable only through Monte Carlo simulation, 
high-dimensional numerical integration, or expressions involving high-dimensional multivariate normal cumulative distribution functions (CDFs).
This significantly restricts the set of applications in which a naive implementation can solve the inner problem faster than a single evaluation of the outer optimization problem.
Stymied by this difficulty,~\cite{GiLeCa08} and later work~\citep{chevalier2013fast}, 
propose heuristic methods that are {\textit motivated} by the one-step optimal algorithm 
of evaluating the set of points that jointly maximize the $\qEI$, but that do not 
actually achieve this gold standard.

\paragraph{Contributions}

The main contribution of this work is to provide a method that solves the inner optimization 
problem of maximizing the $\qEI$ efficiently, creating a practical and broadly applicable one-step Bayes-optimal algorithm for parallel 
global optimization  of expensive functions. To accomplish this we use infinitesimal
perturbation analysis (IPA) \citep{Ho1987} to construct a stochastic gradient
estimator of the gradient of the $\qEI$ surface, and show that this estimator is
unbiased, with a bounded second moment.  Our method uses this estimator within a stochastic gradient ascent
algorithm, which we show converges to the set of stationary points of the $\qEI$ surface.
We use multiple restarts to identify multiple stationary
points, and then select the best stationary
point found.  As the number of restarts and the number of iterations of
stochastic gradient ascent within each restart both grow large, the one-step
optimal set of points to evaluate is recovered. 

Our method can be implemented in both synchronous environments, in which function evaluations are performed in batches and finish at the same time, and asynchronous ones, in which a function evaluation may finish before others are done.

In addition to our methodological contribution, we have 
developed a high-quality open source software package, the ``Metrics Optimization Engine (MOE)'' \citep{moe-github2015},
implementing our method for solving the inner optimization problem and the 
resulting algorithm for parallel global optimization of expensive functions. 
To further enhance computational speed, the implementation takes advantage of
parallel computing and achieves 100X speedup over single-threaded computation when deployed on a graphical processing unit (GPU).
This software package has been used by Yelp and Netflix to solve global optimization problems arising in their businesses \citep{YelpBlog2014,NetflixMLConf2014}.
For the rest of the paper, we refer to our method as \enquote{MOE-qEI} because it is  implemented in MOE.

We compare MOE-qEI against several benchmark methods.  We show that MOE-qEI provides high-quality solutions to the outer optimization problem using fewer function evaluations than the heuristic CL-mix policy proposed by~\cite{chevalier2013fast}, which is motivated by the inner optimization problem.  We also show that MOE-qEI provides a substantial parallel speedup over the single-threaded EGO algorithm, which is one-step optimal when parallel resources are unavailable.
We also compare our simulation-based method for solving the inner optimization problem against methods based on exact evaluation of the $\qEI$ from~\cite{chevalier2013fast} and~\cite{marmin2015differentiating} (discussed in more detail below) and show that our simulation-based approach to solving the inner optimization problem provides solutions to both the inner and outer optimization problem that are comparable in quality and speed when $q$ is small, and superior when $q$ is large.

\paragraph{Related Work}
Developed independently and in parallel with our work is~\cite{chevalier2013fast}, which provides a closed-form formula for computing $\qEI$, and the book chapter~\cite{marmin2015differentiating}, which provides a closed-form expression for its gradient.  Both require multiple calls to high-dimensional multivariate normal CDFs.  These expressions can be used within an existing continuous optimization algorithm to solve the inner optimization problem that we consider.

While attractive in that they provide closed-form expressions, calculating these expressions when $q$ is even moderately large is slow and numerically challenging.  This is because calculating the multivariate normal CDF in moderately large dimension is itself challenging, with state of the art methods relying on numerical integration or Monte Carlo sampling as described in~\cite{genz1992numerical}. Indeed, the method for evaluating the $\qEI$ from~\cite{chevalier2013fast} requires $q^2$ evaluations of the $q-1$ dimensional multivariate normal CDF, and the method for evaluating its gradient requires $O(q^4)$ calls to multivariate normal CDFs with dimension ranging from $q-3$ to $q$. 
In our numerical experiments, we demonstrate that our method for solving the inner optimization problem requires less computation time and parallelizes more easily than do these competing methods for $q>4$, and performs comparably when $q$ is smaller.
We also demonstrate that MOE-qEI's improved performance in the inner optimization problem for $q>4$ translates to improved performance in the outer optimization problem.

Other related work includes the previously proposed heuristic CL-mix from~\cite{chevalier2013fast}, which does not solve the inner maximization of $\qEI$, instead using an approximation. While solving the inner maximization of $\qEI$ as we do makes it more expensive to compute the set of points to evaluate next, we show in our numerical experiments that it results in a substantial savings in the number of evaluations required to find a point with a desired quality.  When function evaluations are expensive, this results in a substantial reduction in overall time to reach an approximately optimal solution.

In other related work on parallel Bayesian optimization,
\cite{frazier2011value} and~\cite{xie2013bayesian} proposed a Bayesian optimization algorithm that 
evaluates pairs of points in parallel, and is one-step Bayes-optimal in the noisy setting under the assumption that one can only observe noisy function values for single points, or noisy function value differences between pairs of points.  This algorithm, however, is limited to evaluating  pairs of points, and does not extend to a higher level of parallelism. 

There are also other non-Bayesian algorithms for derivative-free global optimization of expensive functions with parallel function evaluations 
from~\cite{dennis1991direct, kennedy2010particle} and~\cite{john1992adaptation}.
These are quite different in spirit from the algorithm we develop, not being derived from a decision-theoretic foundation.



\paragraph{Outline}
We begin in Section~\ref{sec:problem} by describing the mathematical setting 
in which Bayesian global optimization is performed, and then defining the $\qEI$ and the 
one-step optimal algorithm.  
We construct our stochastic gradient estimator in Section~\ref{sec:gradient_estimator}, 
and use it within stochastic gradient ascent to define a one-step optimal method 
for parallel Bayesian global optimization in Section~\ref{sec:optimization_qEI}.
Then in Section~\ref{sec:unbiasedness} we show that the constructed gradient estimator of the $\qEI$ surface 
is unbiased under mild regularity conditions, and in Section~\ref{sec:SGA_converge} we provide 
convergence analysis of the stochastic gradient ascent algorithm.
Finally, in Section~\ref{sec:numerical} we present numerical experiments: 
we compare MOE-qEI against previously proposed 
heuristics from the literature; we demonstrate that MOE-qEI provides a speedup 
over single-threaded EGO; we show that MOE-qEI is more efficient than 
optimizing evaluations of the $\qEI$ using closed-form formula provided in~\cite{chevalier2013fast} when $q$ is large;
and we show that MOE-qEI computes the gradient of $\qEI$ faster than evaluating the closed-form expression
proposed in~\cite{marmin2015differentiating}.



\section{Problem formulation and background} 
\label{sec:problem}
In this section, we describe a decision-theoretic approach to Bayesian global
optimization in parallel computing environments, previously proposed by~\cite{GiLeCa08}.  
This approach was considered to be purely conceptual 
as it contains a difficult-to-solve optimization sub-problem (our so-called ``inner'' optimization problem).
In this section, we present this inner optimization problem as background, and present a novel method in the subsequent section that solves it efficiently.

\subsection{Bayesian Global Optimization}
Bayesian global optimization considers optimization of a function $f$ with domain 
$\mathbb{A} \subseteq \mathbb{R}^{d}$.  The overarching goal is to find an approximate solution to 
\begin{equation*}
\min_{\bm{x} \in \mathbb{A}} f(\bm{x}).
\end{equation*}

We suppose that evaluating $f$ is expensive or time-consuming, and that these evaluations 
provide only the value of $f$ at the evaluated point and not its gradient or Hessian. 
We assume that the function defining the domain $\mathbb{A}$ is easy to evaluate and that 
projections from $\mathbb{R}^d$ into the nearest point in $\mathbb{A}$ can be performed quickly.

Rather than focusing on asymptotic performance as the number of function evaluations grows large, 
we wish to find an algorithm that performs well, on average, given a limited budget of function 
evaluations.  To formalize this, we model our prior beliefs on the function $f$ with a Bayesian 
prior distribution, and we suppose that $f$ was drawn at random by nature from this prior 
distribution, before any evaluations were performed.  We then seek to develop an optimization 
algorithm that will perform well, on average, when applied to a function drawn at random in this way.

\subsection{Gaussian process priors}    \label{sec:gp_model}
For our Bayesian prior distribution on $f$, we adopt a Gaussian process prior (see~\citealt{RaWi06}), 
which is specified by its mean function $\mu(\bm{x}) : \mathbb A \rightarrow \mathbb{R}$ 
and positive semi-definite covariance function $k(\bm{x}, \bm{x}') : \mathbb A \times \mathbb A \rightarrow \mathbb{R}$. 
We write the Gaussian process as 
\begin{equation*}
  f \sim \mathcal{GP}(\mu, k).
\end{equation*}
Then for a collection of points $\bm{X} := (\bx_1,\ldots,\bx_q)$, the prior of $f$ at 
$\bm{X}$ is
\begin{equation}
  f(\bm{X}) \sim \mathcal{N} (\bmu{0}, \Sig{0}),
  \label{eq:prior}
\end{equation}
where $\bmu{0}_i = \mu(\bx_i)$ and $\Sig{0}_{ij} = k(\bx_i, \bx_j), i,j \in \{1,\ldots, q\}$. 

Our proposed method for choosing the points to evaluate next additionally requires
that $\mu$ and $k$ satisfy some mild regularity assumptions discussed below, but otherwise
adds no additional requirements.
In practice, $\mu$ and $k$ are
typically chosen using an empirical Bayes approach discussed in~\cite{brochu2010bayesian}, in which
first, a parameterized functional form for $\mu$ and $k$ is assumed; 
second, a first stage of data is collected in which $f$ is evaluated at points chosen 
according to a Latin hypercube or uniform design; 
and third, maximum likelihood estimates for the parameters specifying $m$ and $k$ are obtained.  
In some implementations~\citep{JoScWe98, snoek2012practical}, these estimates are updated iteratively as more evaluations of
$f$ are obtained, which provides more accurate inference and tends to reduce the number of function evaluations required to find good solutions but increases the computational overhead per evaluation.
We adopt this method below in our numerical experiments in
Section~\ref{sec:numerical}.
However, the specific contribution of this paper, a new method for solving an
optimization sub-problem arising in the choice of design points, works with any
choice of mean function $\mu$ and covariance matrix $k$, as long as they satisfy
the mild regularity conditions discussed below.

%




In addition to the prior distribution specified in \eqref{eq:prior}, we may also have some
previously observed function values $\y{i}=f(\x{i})$, for $i=1,\ldots,n$.
These might have been obtained through the previously mentioned first stage of
sampling, running the second stage sampling method we are about to describe,
or from some additional runs of the expensive objective function $f$ performed
by another party outside of the control of our algorithm.  If no additional
function values are available, we set $n=0$.
We define notation $\x{1:n} = (\x{1},\ldots,\x{n})$ and $\y{1:n} = (\y{1},\ldots,\y{n})$.
We require that all points in $\x{1:n}$ be distinct.

We then combine these previously observed function values with our prior to
obtain a posterior distribution on $f(\bm{X})$.  This posterior distribution is still a
multivariate normal (e.g., see Eq.~(A.6) on pp.~200 in~\citealt{RaWi06})
\begin{equation}
    f(\bm{X}) \mid \bm{X}, \x{1:n}, \y{1:n} \sim \mathcal{N} (\bmu{n}, \Sig{n}),
\label{eq:posterior}
\end{equation}
with
\begin{equation}
\begin{split}
&\bmu{n} = \bmu{0} + K\left(\bm{X},\x{1:n}\right) K\left(\x{1:n},\x{1:n}\right)^{-1}  \left(\y{1:n} - \mu(\x{1:n}) \right), \\
&\Sig{n} = K\left(\bm{X},\bm{X}\right) - K\left(\bm{X},\x{1:n}\right)  K\left(\x{1:n},\x{1:n}\right)^{-1} K\left(\x{1:n},\bm{X}\right),
\end{split}
\label{eq:posterior_detail}
\end{equation}
where $\mu(\x{1:n})$ is the vector obtained by evaluating the prior mean function at each point in $\x{1:n}$, $K\left(\bm{X}, \bm{x}^{(1:n)}\right)$ is a $q \times n$ matrix with 
$K\left(\bm{X}, \bm{x}^{(1:n)}\right)_{ij} = k(\bm{x}_i, \bm{x}^{(j)})$, and similarly for 
$K\left(\bm{x}^{(1:n)}, \bm{X}\right)$, $K\left(\bm{X}, \bm{X}\right)$ and $K\left(\bm{x}^{(1:n)}, \bm{x}^{(1:n)}\right)$.


\subsection{Multi-points expected improvement ($\qEI$)}
In a parallel computing environment, we wish to use this posterior distribution
to choose the set of points to evaluate next. \cite{GiLeCa08}
proposed making this choice using a decision-theoretic approach that
considers the utility provided by evaluating a particular candidate set of
points in terms of their ability to reveal better solutions 
than previously known.  We review this decision-theoretic approach here, and then present a new algorithm for implementing this choice in the next section. 

Let $q$ be the number of function evaluations we will perform in parallel,
and let $\bm{X}$ be a candidate 
set of points that we are considering evaluating next. 
Let $f_{n}^{\star} = \min_{m \leq n} f(\x{m})$ 
indicate the value of the best point evaluated, before beginning these $q$ 
new function evaluations. The value of the best point evaluated after all $q$ 
function evaluations are complete will be $\min\left(f_n^{\star},\min_{i = 1,\ldots,q} f(\bm{x}_i )\right)$.
The difference between these two values (the values of the best point evaluated, 
before and after these $q$ new function evaluations) is called the {\textit improvement}, 
and is equal to $\left( f_{n}^{\star} - \min_{i = 1,\ldots,q} f(\bm{x}_i ) \right)^{+}$, 
where $a^+ = \max(a,0)$ for $a\in\R$.

We then compute the expectation of this improvement over the joint probability distribution over $f(\bm{x}_i), i=1,\ldots,q$,
and we refer to this quantity as the 
{\textit multi-points expected improvement} or $\qEI$ from~\cite{GiLeCa08}.  This multi-points 
expected improvement can be written as,
\begin{equation}
\qEI (\bm{X} )=\mathbb{E}_n\left[\left( f_{n}^{\star} - \min_{i = 1,\ldots,q} f(\bm{x}_i ) \right)^{+} \right],
\label{eq:multiEI}
\end{equation}
where $\mathbb{E}_n \left[ \cdot \right] := \mathbb{E} \left[\cdot |\x{1:n}, \y{1:n} \right]$ 
is the expectation taken with respect to the posterior distribution.

\cite{GiLeCa08} then proposes evaluating next the set of 
points that maximize the multi-points expected improvement, 
\begin{equation} 
  \argmax_{\bm{X} \in H} \qEI (\bm{X}),
  \label{eq:maxEI}
\end{equation}
where $H = \{(\bm{x}_1, \ldots, \bm{x}_q) : \bm{x}_i \in \mathbb{A}, \lvert\lvert \bm{x}_i - \bm{x}_j \rvert\rvert \geq r, \lvert\lvert \bm{x}_i - \bm{x}^{(\ell)} \rvert\rvert \geq r, i \neq j, 1 \leq i,j \leq q,  1 \leq \ell \leq n\}$.

This formulation generalizes \cite{GiLeCa08} slightly by allowing an optional requirement that new evaluation points be a distance of at least $r\ge0$ from each other and previously evaluated points.
\cite{GiLeCa08} implicitly set $r=0$.
  Our convergence proof requires $r>0$, which provides a compact feasible domain over which the stochastic gradient estimator has bounded variance.  Setting a strictly positive $r$ can also improve numerical stability in inference (see, e.g., \citealt{ababou1994condition}), and evaluating a point extremely close to a previously evaluated point is typically unlikely to provide substantial improvement in the revealed objective value.
In our experiments we set $r=10^{-5}$.

In the special case $q=1$, which occurs when we are operating without parallelism, 
the multi-points expected improvement reduces to the expected improvement 
\citep{mockus1989bayesian,JoScWe98}, which can be evaluated in closed-form in terms 
of the normal density and CDF as discussed in Section~\ref{sec:intro}. 
\cite{GiLeCa08} provided an analytical expression for $\qEI$ when $q=2$, but in the same paper the authors
commented that computing $\qEI$ for $q>2$ involves 
expensive-to-compute $q$-dimensional Gaussian cumulative distribution functions relying on multivariate integral approximation, which makes solving \eqref{eq:maxEI} difficult.  
\cite{Ginsbourger2009} writes \enquote{directly optimizing the $\qEI$ becomes extremely expensive 
as $q$ and $d$ (the dimension of inputs) grow.}

\section{Algorithm} \label{sec:algorithm}
In this section we present a new algorithm for solving the inner optimization problem  \eqref{eq:maxEI} of maximizing $\qEI$. 
This algorithm uses a novel estimator of the gradient of the $\qEI$ presented in Section~\ref{sec:gradient_estimator},
used within a multistart stochastic gradient ascent framework as described in  Section~\ref{sec:optimization_qEI}.
We additionally generalize this technique from synchronous to asynchronous parallel optimization in Section~\ref{sec:async}. 
Section~\ref{sec:notation} begins by introducing additional notation used to describe our algorithm.

While we keep $q$ fixed (typically to the maximum level of parallelism in one's computing environment), there are settings where one may wish to choose it in a more refined way: 
if parallelism can be flexibly purchased in a cloud computing environment; or if parallel resources can be devoted to parallelizing each function evaluation in addition to running multiple function evaluations in parallel. The Appendix briefly discusses choosing $q$ in these settings.

\subsection{Notation}
\label{sec:notation}
In this section we define additional notation to better support construction of the gradient estimator.
Justified by \eqref{eq:posterior}, we 
write $f(\bm{X})$ as
\begin{equation} 
f(\bm{X}) \,{\buildrel d \over =}\, \bm{\mu}(\bm{X}) +\bm{L}(\bm{X}) \bm{Z},
\label{eq:Y2}
\end{equation}
where $\bm{L}(\bm{X})$ is the lower triangular matrix obtained from the Cholesky decomposition of $\Sig{n}$
in \eqref{eq:posterior}, $\bm{\mu}(\bm{X})$ is the posterior mean (identical to $\bmu{n}$ in \eqref{eq:posterior}, but rewritten here to emphasize the dependence on $X$ and de-emphasize the dependence on $n$),
and $\bm{Z}$ is a multivariate standard normal random vector. 
We will also use the notation $\bm{\Sigma}(\bm{X})$ in place of $\bm{\Sigma}^{(n)}$ in our analysis.

By substituting \eqref{eq:Y2} into \eqref{eq:multiEI}, we have
\begin{equation}
  \qEI (\bm{X}) =\mathbb{E} \left[\left( f_n^* - \min_{i=1,\ldots,q} \bm{e}_i \left[ \bm{\mu}(\bm{X})+\bm{L}(\bm{X}) \bm{Z} \right] \right)^+ \right],
  \label{eq:EI_n1}
\end{equation}
where $\bm{e}_i$ is a unit vector in direction $i$ and the expectation is over $\bm{Z}$. To make \eqref{eq:EI_n1} even more compact, define a new vector $\bm{m}(\bm{X})$
and new matrix $\bm{C}(\bm{X})$, 
\begin{equation}
\begin{split}
  \bm{m}(\bm{X})_i&=
\begin{dcases}
  f_n^* - \bm{\mu}(\bm{X})_i  & \text{if $i>0$ ,} \\
0 & \text{if $i=0$ ,}
\end{dcases}
\\
\bm{C}(\bm{X})_{ij} &=\begin{dcases}
-\bm{L}(\bm{X})_{ij} & \text{if $i>0$ ,} \\
0       & \text{if $i=0$ ,}
\end{dcases}
\end{split}
\label{eq:multiEI_detail}
\end{equation}
and \eqref{eq:EI_n1} becomes
\begin{equation} \label{eq:qEI}
\qEI(\bm{X}) = \mathbb{E} \left[ \max_{i=0,\ldots,q} \bm{e}_i \left[ \bm{m}(\bm{X}) +\bm{C}(\bm{X}) \bm{Z} \right] \right].
\end{equation}

\subsection{Constructing the gradient estimator} \label{sec:gradient_estimator}
We now construct our estimator of the gradient $\nabla \qEI(\bm{X})$.
Let 
\begin{equation} \label{eq:def_f}
h(\bm{X}, \bm{Z}) = \max_{i=0,\ldots,q} \bm{e}_i \left[ \bm{m}(\bm{X}) +\bm{C}(\bm{X}) \bm{Z} \right].
\end{equation}
Then
\begin{equation}
\nabla \qEI(\bm{X}) = \nabla \mathbb{E} h(\bm{X}, \bm{Z}).
\label{eq:grad_qei}
\end{equation}
If gradient and expectation in \eqref{eq:grad_qei} are interchangeable,
the gradient would be
\begin{equation}
\nabla \qEI(\bm{X}) = \mathbb{E}  \bm{g}(\bm{X}, \bm{Z}),
\label{eq:unbiasedness}
\end{equation}
where
\begin{equation} \label{eq:gradient_def}
\bm{g}(\bm{X}, \bm{Z})= \begin{cases} \nabla h(\bm{X},\bm{Z})  & \text{if } \nabla h(\bm{X}, \bm{Z})  \text{ exists,} \\ 0 & \text{otherwise.} \end{cases}
\end{equation}
$\bm{g}(\bm{X}, \bm{Z})$ can be computed using results on differentiation of the Cholesky decomposition from~\cite{smith1995differentiation}.

We use $\bm{g}(\bm{X}, \bm{Z})$ as our estimator of the gradient $\nabla \qEI$, and will
discuss interchangeability of gradient and expectation, which implies unbiasedness of our gradient estimator, in Section~\ref{sec:unbiasedness}.
As will be discussed in Section~\ref{sec:SGA_converge}, unbiasedness of the gradient estimator is one of the sufficient conditions for convergence
of the stochastic gradient ascent algorithm proposed in Section~\ref{sec:optimization_qEI}.

\subsection{Optimization of $\qEI$} \label{sec:optimization_qEI}
Our stochastic gradient ascent algorithm begins with some
initial point $\bm{X}_0 \in H$,
and generates a sequence 
$\{\bm{X}_t: t=1, 2, \ldots\}$ using
\begin{equation}
  \bm{X}_{t+1} = \prod_{H} \left[ \bm{X}_{t} + \epsilon_t \bm{G}(\bm{X}_t) \right],
  \label{eq:sga}
\end{equation}
where $\prod_{H}(\bm{X})$ denotes the closest point in $H$ to $\bm{X}$,
and if the closest point is not unique, a closest point such that the function $\prod_{H}(\cdot)$ is measurable. 
$\bm{G}(\bm{X}_t)$ is an estimate of the gradient of $\qEI(\cdot)$ at $\bm{X}_t$, 
obtained by averaging 
$M$ replicates of our stochastic gradient estimator,
\begin{equation}
  \bm{G}(\bm{X}_t) = \frac{1}{M} \sum_{m=1}^M \bm{g}(\bm{X}_t, \bm{Z}_{t,m}),
  \label{eq:G_eval}
\end{equation}
where \{$\bm{Z}_{t,m}$: m=1, \dots, M\} are i.i.d. samples generated from the multivariate standard normal distribution,
$\bm{g}(\bm{X}_t, \bm{Z}_{t,m})$ is defined in \eqref{eq:gradient_def}.
$\{\epsilon_t: t=0, 1, \ldots\}$ is a stochastic gradient stepsize sequence 
\citep{kushner2003stochastic}, typically chosen to be equal to $\epsilon_t = \frac{a}{t^{\gamma}}$
for some scalar $a$ and $\gamma\in(0,1]$.  Because we use Polyak-Ruppert averaging as described below, we set $\gamma<1$.
Analysis in Section~\ref{sec:theory} shows that, under certain mild conditions, 
this stochastic gradient algorithm 
converges almost surely to the set of stationary points. 

After running $T$ iterations of stochastic gradient ascent using \eqref{eq:sga}, we obtain the sequence $\{\bm{X}_t : t=1,2,\ldots,T\}$.
From this sequence we extract the average $\overline{X}_T = \frac1{T+1} \sum_{t=0}^T \bm{X}_t$ and use it as an estimated stationary point.   This Polyak-Ruppert averaging approach \citep{polyak1990new,ruppert1988efficient} is more robust to misspecification of the stepsize sequence than using $\bm{X}_T$ directly.

To find the global maximum of the $\qEI$, we use multiple restarts of the 
algorithm from a set of starting points, drawn from a Latin
hypercube design \citep{mckay2000comparison}, to find multiple stationary points,
and then use simulation
to evaluate $\qEI$ at these stationary points and select the point for which it is largest.
For simplicity we present our approach using a fixed sample size $N$ to perform this evaluation and selection (Step~\ref{step:estimate} in Algorithm~\ref{algo1} below) 
but one one could also use a more sophisticated ranking and selection algorithm with adaptive sample sizes (see, e.g., \citealt{kim2007recent}), or evaluate $\qEI$ using the closed-form formula in \cite{chevalier2013fast}.
We summarize our procedure for selecting the set of points to sample next, which we call MOE-qEI, in Algorithm~\ref{algo1}.

\begin{algorithm}
\caption{MOE-qEI: Optimization of $\qEI$}\label{algo1}
\begin{algorithmic}[1]
\REQUIRE number of starting points $R$; stepsize constants $a$
and $\gamma$; number of steps for one run of gradient ascent $T$;  
number of Monte Carlo samples for estimating the gradient $M$; number of 
Monte Carlo samples for estimating $\qEI$ $N$.
\STATE Draw $R$ starting points from a Latin hypercube design in $H$, $\bm{X}_{r,0}$ for $r=1,\ldots,R$ .
\FOR{$r=1$ to $R$} 
\FOR{$t=0$ to $T-1$}
\STATE Compute $\bm{G}_t = \frac{1}{M} \sum_{m=1}^M \bm{g}(\bm{X}_{r,t},\bm{Z}_{r,t,m})$ 
where $\bm{Z}_{r,t,m}$ is a vector of $q$ i.i.d. samples drawn from the standard normal distribution.
\STATE Update solution using stochastic gradient ascent $\bm{X}_{r,t+1} = \prod_{H} \left[ \bm{X}_{r,t} + \frac{a}{t^{\gamma}} \bm{G}_t \right]$.
\ENDFOR
\label{step:polyak-ruppert}
\STATE Compute the simple average of the solutions for $\bm{X}_{r,t}$, $\overline{\bm{X}}_{r,T} = \frac{1}{T+1} \sum_{t=0}^T\bm{X}_{r, t}$. \label{step:done_start}
\STATE Estimate $\qEI (\overline{\bm{X}}_{r,T})$ using Monte Carlo simulation with $N$ i.i.d. samples, and store the estimate as $\widehat{\qEI}_r$.  \label{step:estimate}
\ENDFOR
\RETURN $\overline{\bm{X}}_{r',T}$ where $r' = \argmax_{r=1,\ldots,R} \widehat{\qEI_r}$.
\end{algorithmic}
\end{algorithm}

The MOE software package \citep{moe-github2015} implements Algorithm~\ref{algo1},
and supplies the following additional optional fallback logic.  
If $\max_{r=1,\ldots,R} \widehat{\qEI}_r \le \epsilon'$, 
so that multistart stochastic gradient ascent fails to find a point with estimated 
expected improvement better than $\epsilon'$, then it generates $L$ additional 
solutions from a Latin Hypercube on $H$, estimates the $\qEI$
at each of these using the same Monte Carlo approach as in Step~\ref{step:estimate}, 
and selects the one with the largest estimated $\qEI$.
This logic takes two additional parameters: a strictly positive real number 
$\epsilon'$ and an integer $L$.  
We turn this logic off in our experiments by setting $\epsilon'=0$.

\subsection{Asynchronous parallel optimization} \label{sec:async}
So far we have assumed synchronous parallel optimization, in which we wait for all $q$ points to finish before choosing a new set of points.
However, in some applications, we may wish to generate a new partial batch of points to evaluate next while $p$ points are still being evaluated, before we have 
their values. This is common in expensive computer evaluations, which do not necessarily finish at the same time.

We can extend Algorithm~\ref{algo1} to solve an extension of 
\eqref{eq:maxEI}  proposed by \cite{ginsbourger2010kriging} for 
asynchronous parallel optimization: suppose parallelization
allows a batch of $q$ points to be evaluated simultaneously; the first $p$ points are
still under evaluation, while the remaining $q-p$ points have finished evaluation and the resources used to evaluate them are free to evaluate new points. We let $\bm{X'} := (\bm{x}_1, \ldots, \bm{x}_p)$ be the first
$p$ points still under evaluation, and let $\bm{X}:=(\bm{x}_{p+1}, \ldots, \bm{x}_q)$
be the $(q-p)$ points ready for new evaluations. Computation of $\qEI$ for these $q$ points 
remains the same as in \eqref{eq:multiEI},
but we use an alternative notation, 
$\qEI(\bm{X'}, \bm{X})$, to explicitly indicate that $\bm{X}'$ are the points still being evaluated and $\bm{X}$ are the new points to evaluate. 
We emphasize that the expectation in $\qEI(\bm{X'},\bm{X})$ is over yet-to-be-observed values of $f$ at both $\mathbf{X}$ and $\mathbf{X'}$, and is taken with respect to the posterior given the previous $n$ evaluations.
Keeping $\bm{X}'$ fixed, we optimize $\qEI$ over $\bm{X}$ by
solving this alternative problem
\begin{equation} \label{eq:qpEI}
\argmax_{\bm{X} \in H'} \qEI (\bm{X}', \bm{X}),
\end{equation}
where $H' = \{(\bm{x}_{p+1}, \ldots, \bm{x}_{q}): \bm{x}_i \in \mathbb{A}, \lvert\lvert \bm{x}_i - \bm{x}_j \rvert\rvert \geq r, \lvert\lvert \bm{x}_i - \bm{x}_k \rvert\rvert \geq r, \lvert\lvert \bm{x}_i - \bm{x}^{(m)} \rvert\rvert \geq r, i \neq j, p < i \leq q, p < j \leq q, 1 \leq k \leq p, 1 \leq m \leq n \}$ for some small positive $r$.
As we did in the algorithm for synchronous parallel optimization in Section~\ref{sec:optimization_qEI},
we estimate the gradient of the objective function with respect to $\bm{X}$, i.e., $\nabla_{\bm{X}} \qEI (\bm{X}', \bm{X})$. 
The gradient estimator is essentially the same as that
in Section~\ref{sec:gradient_estimator}, except that we only
differentiate $h(\cdot, \cdot)$ with respect to $\bm{X}$. 
(Although $h(\cdot,\cdot)$ is only differentiated with respect to $\mathbf{X}$, it depends on both $\mathbf{X}$ and $\mathbf{X'}$.)
Then we proceed according to Algorithm \ref{algo1}.

In practice, one typically sets $p=q-1$.
This is because 
Bayesian optimization procedures are used most frequently when function evaluation times are large, 
and asynchronous computing environments 
typically have a time between evaluation completions that increases with the evaluation time.
When this inter-completion time is large relative to the time required to solve \eqref{eq:qpEI}, 
it is typically better to solve~\eqref{eq:qpEI} each time an evaluation completes, i.e., to set $p=q-1$.

Indeed, if we set $p < q-1$ then we let CPU cores sit idle for longer and we decrease the total utilization of our parallel computing environment.  For example, consider the CPU core that finishes first.  When $p = q-1$, it waits to start a new function evaluation only for the time it takes to maximize \qEI.  Moreover, this core can be used to do this maximization so that no time is wasted.  But, if $p<q-1$, it must sit idle while we wait for an additional $q-1-p$ cores to complete.  By not letting cores sit idle, we reduce the wall-clock time required to do a given number of function evaluations.

If the time to perform a function evaluation is small enough, or if the computing environment is especially homogeneous, then the time between completions might be substantially smaller than the time to solve~\eqref{eq:qpEI} and one might wish to set $p$ strictly smaller than $q-1$. This may be beneficial because short intercompletion times keep the resulting increase in idle time small and 
it allows results from a larger group of function evaluations to be available when deciding how to next allocate cores.

\section{Theoretical analysis} \label{sec:theory}
In Section~\ref{sec:algorithm},
In Section~\ref{sec:algorithm},
when we constructed our gradient estimator and described the use of stochastic gradient ascent to optimize
$\qEI$, we alluded to conditions under which this gradient estimator is unbiased and this stochastic gradient ascent algorithm converges to the set of stationary points of the $\qEI$ surface.
In this section, we describe these conditions and state these results.

\subsection{Unbiasedness of the gradient estimator} \label{sec:unbiasedness}

We now state our main theorem showing unbiasedness of the gradient estimator. Proofs of all results including supporting lemmas are available as supplemental material.

\begin{theorem} \label{thm_grad}
  If $\bm{m}(\bm{X})$ 
  and $\bm{C}(\bm{X})$ are continuously differentiable 
  in a neighborhood of $\bm{X}$, and $\bm{C}(\bm{X})$ 
  has no duplicate rows, then $\nabla h(\bm{X},\bm{Z})$ 
  exists almost surely and
  \begin{equation*}
    \nabla \mathbb{E} h(\bm{X},\bm{Z}) = \mathbb{E} \nabla h(\bm{X},\bm{Z}).
  \end{equation*} 
\end{theorem}

Theorem~\ref{thm_grad} requires continuous differentiability of $\bm{C}(\bm{X})$, which may seem difficult to verify. However, using \cite{smith1995differentiation}, which shows that $m$th-order differentiability of a 
symmetric and nonnegative definite matrix implies $m$th-order differentiability of the lower triangular matrix obtained from its Cholesky factorization, $\bm{L}(\bm{X})$ and thus $\bm{C}(\bm{X})$ have the same order of differentiability as $\Sig{n}$, whose order of differentiability can in turn be verified by examination of the prior covariance function $k(\cdot,\cdot)$. In addition, when $\Sig{n}$ is positive definite,  $\bm{C}(\bm{X})$ will not have duplicate rows.
We will use these facts below in Corollary~\ref{corollary}, after first discussing convergence, to provide easy-to-verify conditions under which unbiasedness and convergence to the set of stationary points hold.


\subsection{Convergence analysis} \label{sec:SGA_converge}
In this section, we show almost sure convergence of our proposed stochastic gradient ascent algorithm.
We assume that $\mathbb{A}$ is compact and can be written in the form 
$\mathbb{A} = \{\bm{x}: a_i'(\bm{x}) \leq 0, i=1, \ldots, m'\} \subseteq \mathbb{R}^{d}$, where $a_i'(\cdot)$ is any real-valued constraint function.
Then $H$ can be written in a form more convenient for analysis,
\begin{equation*}
H = \{\bm{X}: a_i(\bm{X}) \leq 0, i=1, \ldots, m\} \subseteq \mathbb{R}^{d \times q}, 
\end{equation*}
where $a_{(i-1)q+j}(\bm{X}) = a_i'(\bm{x}_j)$ with $\bm{x}_j$ being the $j$th point in $\bm{X}$,
and $a_i(\bm{X})$ for $i>m'q$ encodes the constraints $\lvert\lvert \bm{x}_i - \bm{x}_j \rvert\rvert \ge r$ and 
$\lvert\lvert \bm{x}_i - \bm{x}^{(\ell)} \rvert\rvert \ge r$ 
present in \eqref{eq:maxEI}.


The following theorem shows that Algorithm~\ref{algo1} converges to the set of stationary points
under conditions that include those of Theorem~\ref{thm_grad}.
The proof is available as supplemental material.
\begin{theorem} \label{thm:sga}
  Suppose the following assumptions hold,
  \begin{enumerate}
    \item 
      $a_i(\cdot), i = 1, \ldots, m$ are continuously differentiable.
    \item
      $\epsilon_t \rightarrow 0$ for $t \geq 0$; $\sum_{t=1}^{\infty} \epsilon_t = \infty$ and $\sum_{t=0}^{\infty} \epsilon_t^2 < \infty$.
    \item
      $\forall \bm{X} \in H$, $\bm{\mu}(\bm{X})$ and $\bm{\Sigma}(\bm{X})$ are twice continuously differentiable and $\bm{\Sigma}(\bm{X})$ 
      is positive definite.
  \end{enumerate}
 Then the sequence $\{\bm{X}_t: t=0,1,\ldots\}$ and its Polyak-Ruppert average $\{\overline{\bm{X}}_t : t=0,1,\ldots\}$ generated by algorithm 
 \eqref{eq:sga} converges almost surely to a connected set of stationary points of the $\qEI$ surface.
\end{theorem}

The following corollary of Theorem~\ref{thm:sga} uses conditions that can be more easily checked prior to running MOE-qEI.  It requires that the sampled points are distinct, which can be made true by dropping duplicate samples.  Since function evaluations are deterministic, no information is lost in doing so. 

\begin{corollary} \label{corollary}
If the sampled points $\x{1:n}$ are distinct and 
\begin{enumerate}
    \item the prior covariance function $k$ is positive definite and twice differentiable,
    \item the prior mean function $\mu$ is twice differentiable,
    \item conditions 1 and 2 in Theorem~\ref{thm:sga} are met,
\end{enumerate}
then $\bm{X}_t$ and its Polyak-Ruppert average $\overline{\bm{X}}_t$ converge to a connected set of stationary points.
\end{corollary}
\begin{proof}{Proof to Corollary~\ref{corollary}.}
Since $\x{1:n}$ are distinct, and $\bm{X} \in H$, and the prior covariance function is positive definite and twice continuously differentiable, then $K\left(\bm{X}, \bm{x}^{(1:n)}\right)$, $K\left(\bm{x}^{(1:n)}, \bm{X}\right)$, $K\left(\bm{X}, \bm{X}\right)$ and $K\left(\bm{x}^{(1:n)}, \bm{x}^{(1:n)}\right)$ in~\eqref{eq:posterior_detail} are all positive definite and twice continuously differentiable. Since
the prior mean function is also twice continuously differentiable, it follows that $\bmu{n}=\bm{\mu}(\bm{X})$ and $\Sig{n}=\bm{\Sigma}(\bm{X})$ defined in~\eqref{eq:posterior_detail} are twice continuously differentiable, and in addition, $\Sig{n}$ is positive definite. Thus the conditions of Theorem~\ref{thm:sga} are verified, and
its conclusion holds.
\end{proof}

\section{Numerical results}
\label{sec:numerical}
In this section, we present numerical experiments demonstrating the performance of MOE-qEI. The implementation of MOE-qEI follows Algorithm~\ref{algo1}, and is available in the open source software package \enquote{MOE}
\citep{moe-github2015}.

We first discuss the choice of constants in Algorithm~\ref{algo1}: $R$, $T$, $M$, $N$, $\gamma$, and $a$.
\begin{enumerate}
\item
Number of starting points, $R$: this should be larger and of the same order as the number of equivalence classes of stationary points of the $\qEI$ surface, where we identify a set of stationary points as in the same class if they can be obtained from each other by permuting $\bm{x}_1,\ldots,\bm{x}_q$.  ($\qEI$ is symmetric and such permutations do not change its value.)
However, we do not know the number of such equivalence classes, and their number tends to grow with $n$ as the surface grows more modes.
Setting $R$ larger increases our chance of finding the global maximum but increases computation. 
In our numerical experiments, we set $R=n$ to capture this trade-off between runtime and solution quality.
As a diagnostic, one can check whether $R$ is large enough by checking the number of unique solutions we obtain; if we obtain the same solution repeatedly from multiple restarts, this suggests $R$ is large enough.
\item
Number of steps in stochastic gradient ascent, $T$, and stepsize sequence parameters $a$ and $\gamma$: For simplicity, we set $a=1$.  We set $\gamma=0.7$, which is significantly below $1$, to ensure that the stepsize sequence decreases slowly, as is recommended when using Polyak-Ruppert averaging.  We then plotted the $\qEI$ and norm of the gradient from stochastic gradient ascent versus $t$ for a few sample problems.  Finding that convergence occurred well before the 100th iterate, we set $T=100$. As a diagnostic, one may also assess convergence by evaluating the gradient using a large number of Monte Carlo samples at the final iterate $T$ and comparing its norm to $0$.
\item
Number of Monte Carlo samples $M$: this determines the accuracy of the gradient estimate and therefore affects stochastic gradient ascent's convergence. 
We set $M=1000$ and as discussed in Section~\ref{sec:optimization_qEI} performed an experiment to justify this setting.
While we ran our experiments on a CPU except where otherwise stated to ensure a fair comparison with other competing algorithms, for which a GPU implementation is not available, 
the \enquote{MOE} software package provides a GPU implementation that can be used to increase the amount of parallelism used in MOE-qEI.
When using the GPU implementation, we recommend setting $M=10^6$ because the GPU's parallelism makes averaging a large number of independent replicates fast, and the reduction in noise reduces the number of iterates needed for convergence by stochastic gradient ascent.
\item
Number of Monte Carlo samples for estimating $\qEI$, $N$: we estimate the $\qEI$ at a limiting solution only once for each restart, i.e., $R$ times, and so setting $N$ large introduces little computational overhead.  We set $N=10^6$ to ensure an essentially noise-free selection of the best of the limiting solutions, and we assess this choice by examining the standard error of our estimates of the $\qEI$.
\end{enumerate}

For the outer optimization of the objective function, we begin with a small dataset, typically sampled using a Latin hypercube design, to train the Gaussian Process model described in Section~\ref{sec:gp_model}. In our numerical experiments, we use $\mu=0$ and a squared exponential kernel $k$ whose hyperparameters are estimated using an empirical Bayes approach: we set them to the values that maximize the log marginal likelihood of the observed data.
With the trained Gaussian Process model, we perform the inner optimization of MOE-qEI described in Algorithm~\ref{algo1} to find the batch of points to evaluate, and after evaluating them we update the hyperparameters as well as the Gaussian Process model. We repeat this process over a number of iterations and report the best solution found in each iteration.

Noise-free function evaluations may often lead to ill-conditioned covariance matrices $K(\cdot,\cdot)$ in~\eqref{eq:posterior}. To resolve this problem, we adopt a standard trick from Gaussian process regression~\cite[Section~3.4.3]{RaWi06}: we manually impose a small amount of noise $\sim \mathcal{N}(0, \sigma^2)$ where $\sigma^2 = 10^{-4}$ and use Gaussian Process regression designed for noisy settings, which is almost identical to \eqref{eq:posterior} except that $K(\x{1:n},\x{1:n})$ is replaced by $K(\x{1:n},\x{1:n}) + \sigma^2 I_n$ where $I_n$ is the identity matrix \cite[Section~2.2]{RaWi06}.

\subsection{Comparison on the outer optimization problem} \label{sec:horserace}
Constant Liar is a heuristic algorithm motivated by \eqref{eq:qEI} proposed by \cite{GiLeCa08}, which uses a greedy approach to iteratively construct a batch 
of $q$ points. At each iteration of this greedy approach, the heuristic uses the sequential EGO algorithm to find a point that maximizes the expected improvement. However, since the posterior used by EGO depends on the current batch of points, which have not yet been evaluated, Constant Liar imposes a heuristic response (the \enquote{liar}) at this point, and updates the Gaussian Process model with this \enquote{liar} value. The algorithm stops when $q$ points are added, and reports the batch for function evaluation.

\begin{figure}[H]
\centering
    \centering
    \subfloat[0][Branin2]{
        \includegraphics[width=0.45\textwidth]{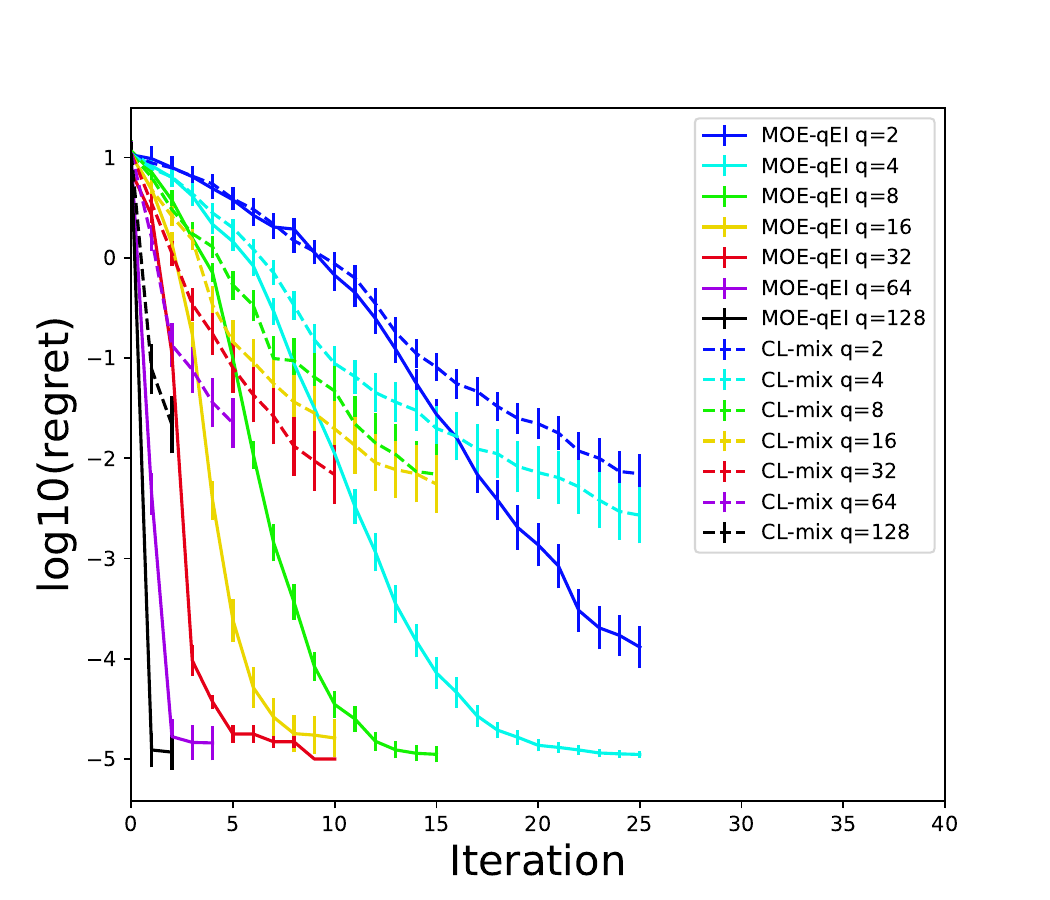}
    }
    \subfloat[1][Hartmann3]{
        \includegraphics[width=0.45\textwidth]{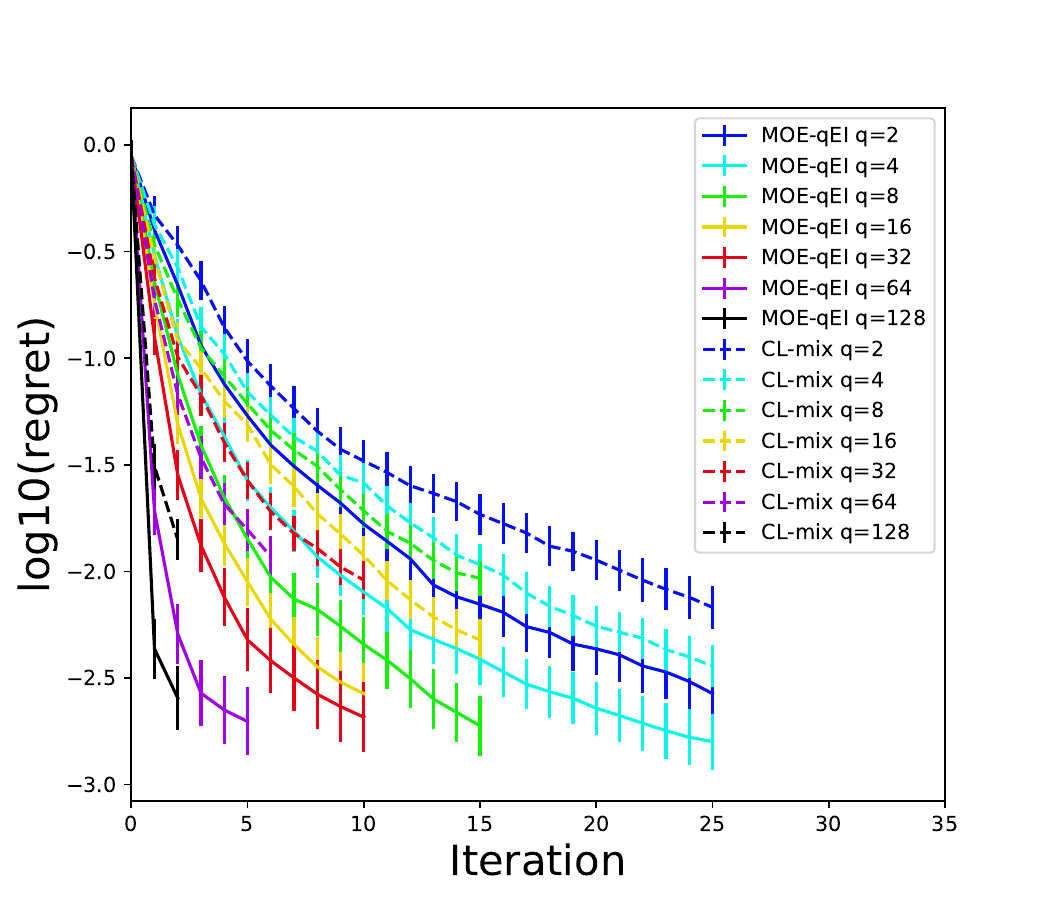}
    }
    \qquad
    \subfloat[2][Ackley5]{
        \includegraphics[width=0.45\textwidth]{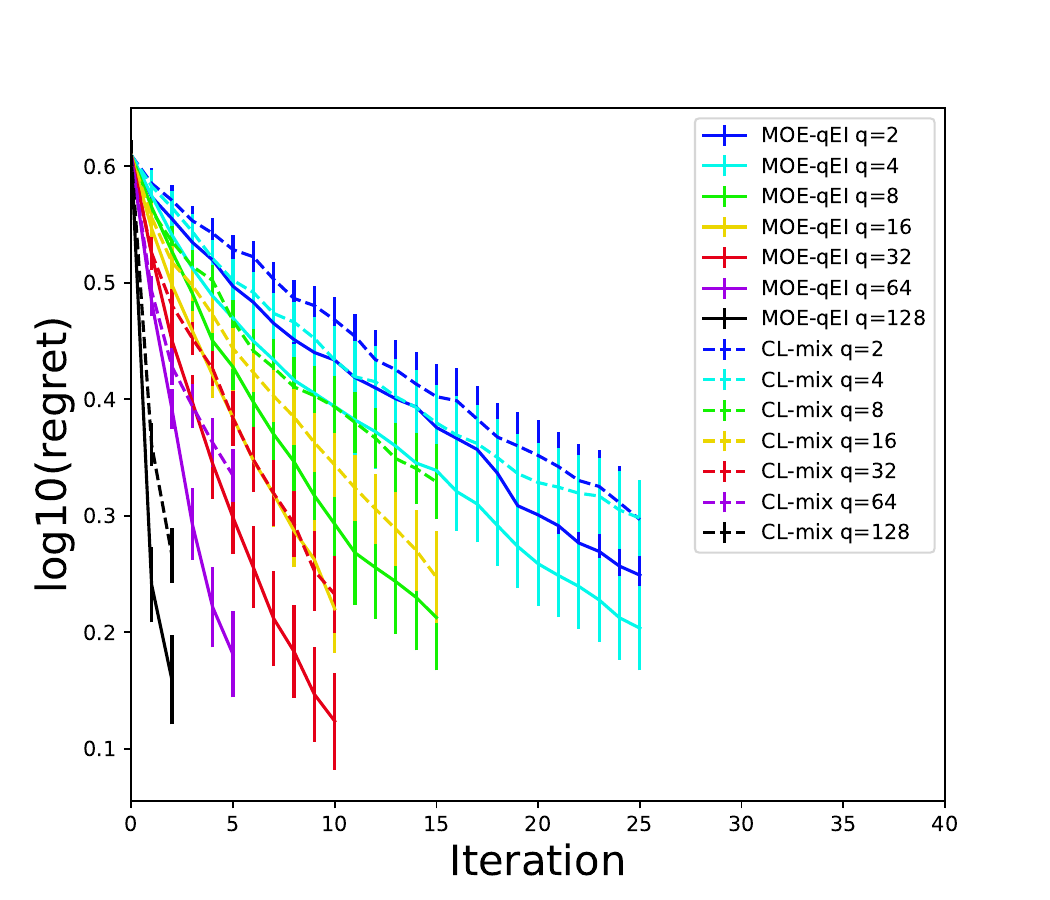}
    }
    \subfloat[3][Hartmann6]{
        \includegraphics[width=0.45\textwidth]{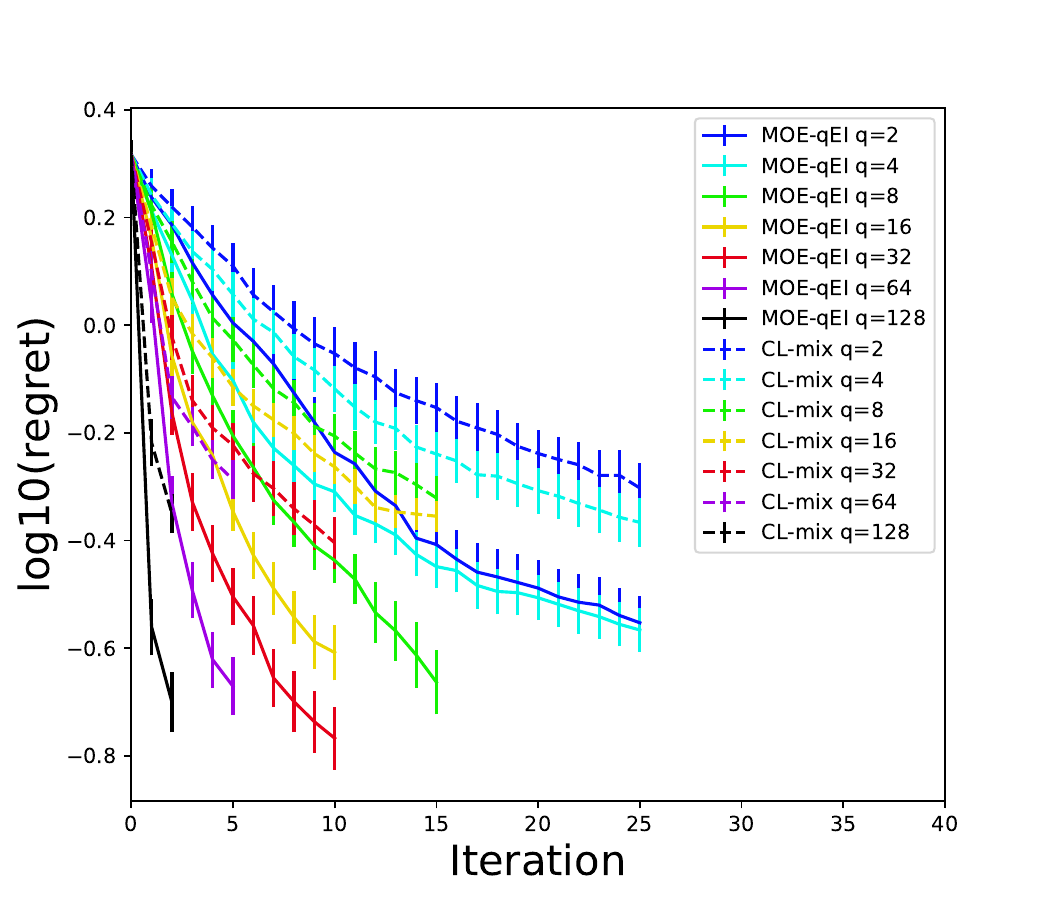}
    }
    \caption{Comparison against Constant Liar: $\log_{10}(\text{regret})$ vs. iteration for MOE-qEI (solid line) and CL-mix (dashed line), where the error bars show 95\% confidence intervals obtained from 100 repeated experiments with different sets of initial points. MOE-qEI converges faster with better solution quality than the  heuristic method CL-mix for all $q$.} 
    \label{fig:benchmark}
\end{figure}

There are three variants of Constant Liar (CL), which use three different strategies for choosing the liar value:
CL-min sets the liar
value to the minimum response observed so far; CL-max sets it to the maximum response observed so far; and CL-mix is a hybrid of
the two, computing one set of points using CL-min, another set of points using CL-max, and sampling the set that has the higher $\qEI$. Among the three methods, CL-mix was shown by \cite{chevalier2013fast} to have the best overall performance, and therefore we compare MOE-qEI against CL-mix.

We ran MOE-qEI and CL-mix on a range of standard test functions for global optimization~\citep{jamil2013literature}: 2-dimensional Branin2; 3-dimensional Hartmann3; 5-dimensional Ackley5; and 6-dimensional Hartmann6.
In the experiment, we first draw $(2d+2)$ points in the domain using a Latin hypercube design, where $d$ is the dimension of the objective function, and fit a Gaussian Process model using the initial points.
Thereafter, we let MOE-qEI and CL-mix optimize over the test functions and report
for each iteration the regret as $\text{regret}=f^*-\text{best solution so far}$ for each iteration, where we note that each of the problems considered is a minimization problem.
 We repeat the experiment 100 times using different initial sets of points, and report the average performance of both algorithms in Figure~\ref{fig:benchmark}. The result shows that MOE-qEI consistently finds better solutions than the heuristic method on all four test functions.
 
Next, we compare MOE-qEI and CL-MIX at different levels of parallelism using the same experimental setup as above.
The sequential EGO algorithm makes the same decisions as MOE-qEI when $q=1$ and so this may also be seen as a comparison against EGO.
Figure~\ref{fig:speedup} shows that MOE-qEI achieves significant speedup over EGO as $q$ grows, indicating substantial potential time saving using parallelization and MOE-qEI in Bayesian optimization tasks.

\begin{figure}[H]
\centering
    \centering
    \subfloat[0][Branin2]{
        \includegraphics[width=0.45\textwidth]{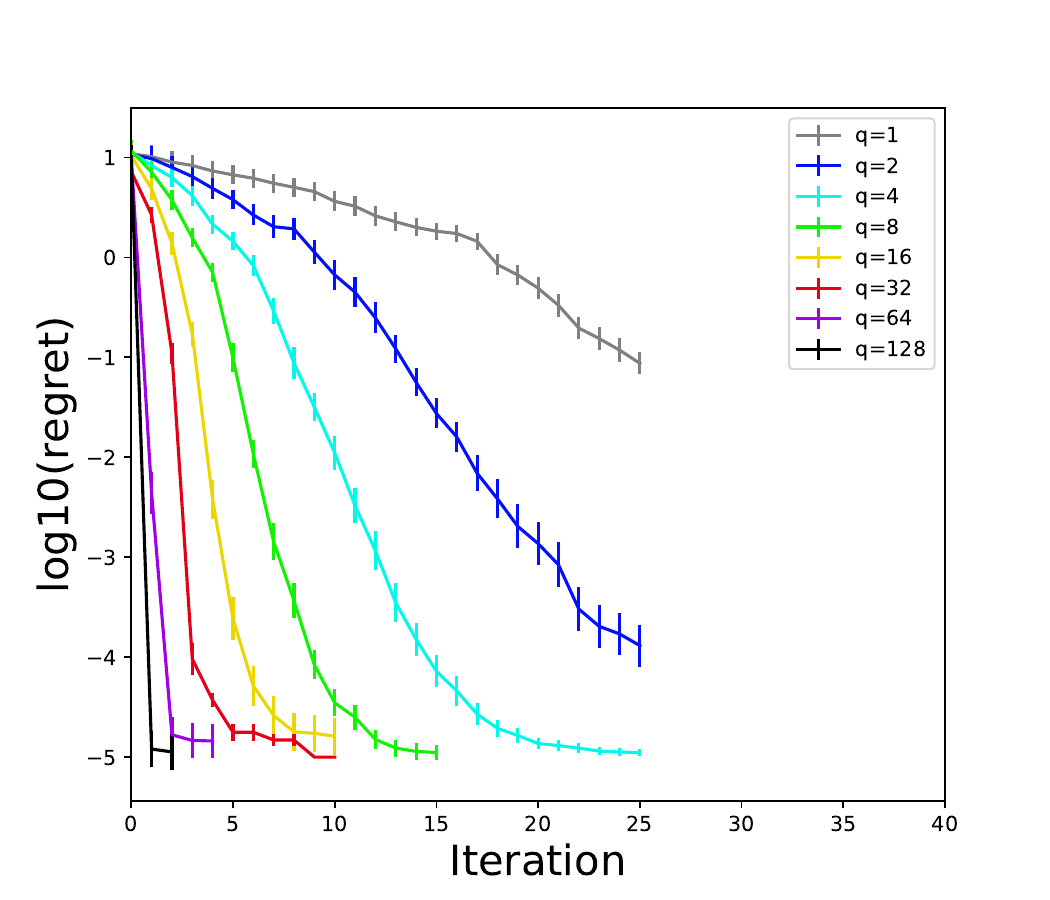}
    }
    \subfloat[1][Hartmann3]{
        \includegraphics[width=0.45\textwidth]{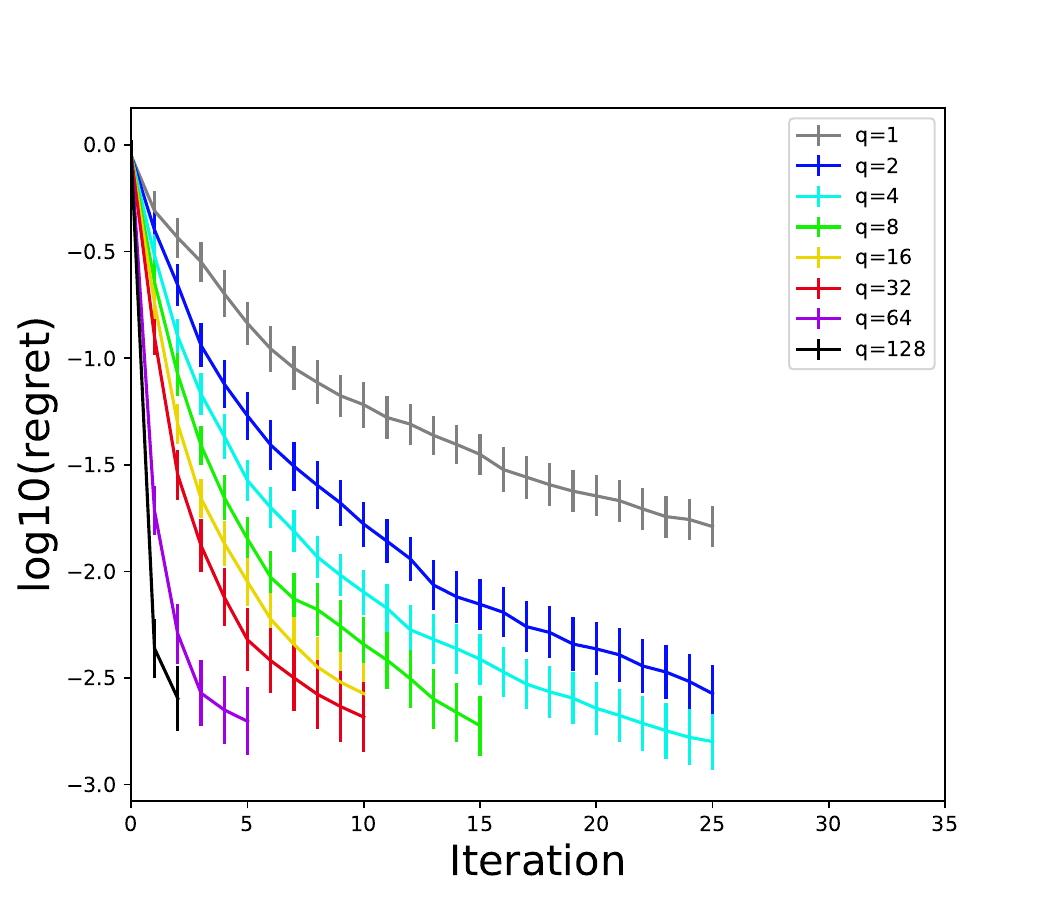}
    }
    \qquad
    \subfloat[2][Ackley5]{
        \includegraphics[width=0.45\textwidth]{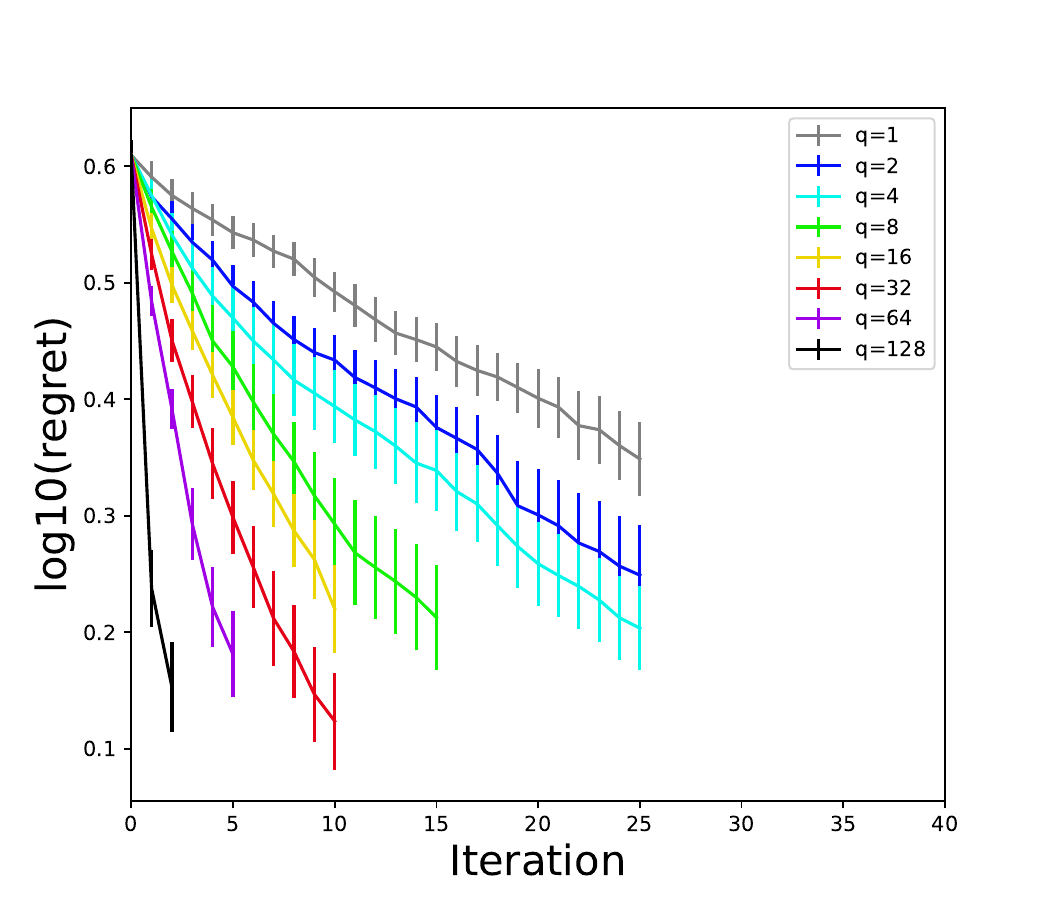}
    }
    \subfloat[3][Hartmann6]{
        \includegraphics[width=0.45\textwidth]{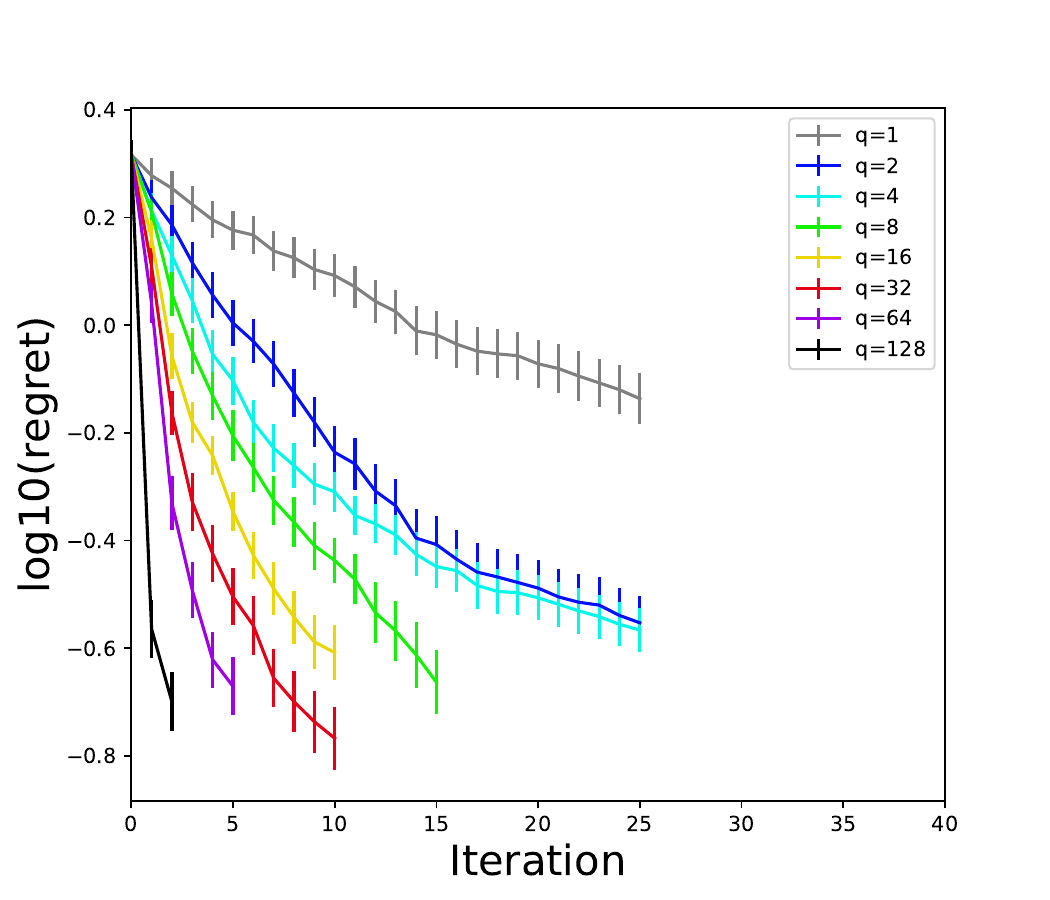}
    }
    \caption{Comparison against EGO: $\log_{10}(\text{regret})$ vs. iteration for different $q$, where the error bars show 95\% confidence intervals obtained from 100 repeated experiments with different sets of initial points.} 
    \label{fig:speedup}

\end{figure}

\subsection{Comparison on the inner optimization problem} \label{sec:exact}
\cite{chevalier2013fast} provided a closed-form formula for $\qEI$ and argued that it computes $\qEI$ 
\enquote{very fast for reasonably low values of $q$ (typically less than 10)}. The closed-form formula is
provided as follows for reference, modified to use the notation in this paper.
Recall~\eqref{eq:posterior} and let $f(\bm{X}) = (Y_1, \ldots, Y_q)$ be a random vector with mean $\bmu{n}$ and covariance matrix $\Sig{n}$. For $k \in \{1, \ldots, q\}$ consider the vectors $\bm{Z}^k := (Z_1^k, \ldots, Z_q^k)$ defined as follows:
\begin{equation*}
\begin{split}
Z_j^k &:= Y_k - Y_j, j \neq k, \\
Z_k^k &:= Y_k.
\end{split}
\end{equation*}
Let $\bm{m}^k$ and $\Sigma^k$ denote the mean and covariance matrix of $\bm{Z}^k$, and define the vector $\bm{b}^k \in \mathbb{R}^q$ by $b_k^k = f_n^*$ and $b_j^k = 0$ if $j \neq k$. Then the closed-form formula is
\begin{equation}
\qEI(\bm{X}) = \sum_{k=1}^q \left( (f_n^* - \mu_k^{(n)}) \Phi_q(\bm{b}^{k}-\bm{m}^{k}, \Sigma^{k}) + \sum_{i=1}^q \Sigma_{ik}^{k} \phi_{m_i^{k}, \Sigma_{ii}^{k}}(b_i^{(k)}) \Phi_{q-1} (\bm{c}_{.i}^{k}, \Sigma_{.i}^{k}) \right),
\label{eq:fast_ei}
\end{equation}
where $\bm{c}^{k}$ is as defined in \cite{chevalier2013fast}.

This formula requires $q$ calls of the $q$-dimensional multivariate normal CDF ($\Phi_q(\cdot, \cdot)$), and $q^2$ calls of the $q-1$ dimensional multivariate normal CDF ($\Phi_{q-1}(\cdot, \cdot)$). Since computing multivariate normal CDFs, which are often implemented with numerical integration or Monte Carlo sampling \citep{genz1992numerical}, is expensive to evaluate even for moderate $q$, calculating this analytically formula quickly becomes slow and numerically challenging as $q$ grows.

While~\cite{chevalier2013fast} did not propose using this closed-form formula to solve the inner optimization problem \eqref{eq:maxEI}, one can adapt it to this purpose by using it within any derivative-free optimization method.  We implemented this approach in the MOE package, where we use the L-BFGS~\citep{liu1989limited} solver from SciPy \citep{scipy} as the derivative free optimization solver. We call this approach \enquote{Benchmark 1}.
\begin{figure}[H]
\centering
    \centering
    \subfloat[0][Branin2]{
        \includegraphics[width=0.45\textwidth]{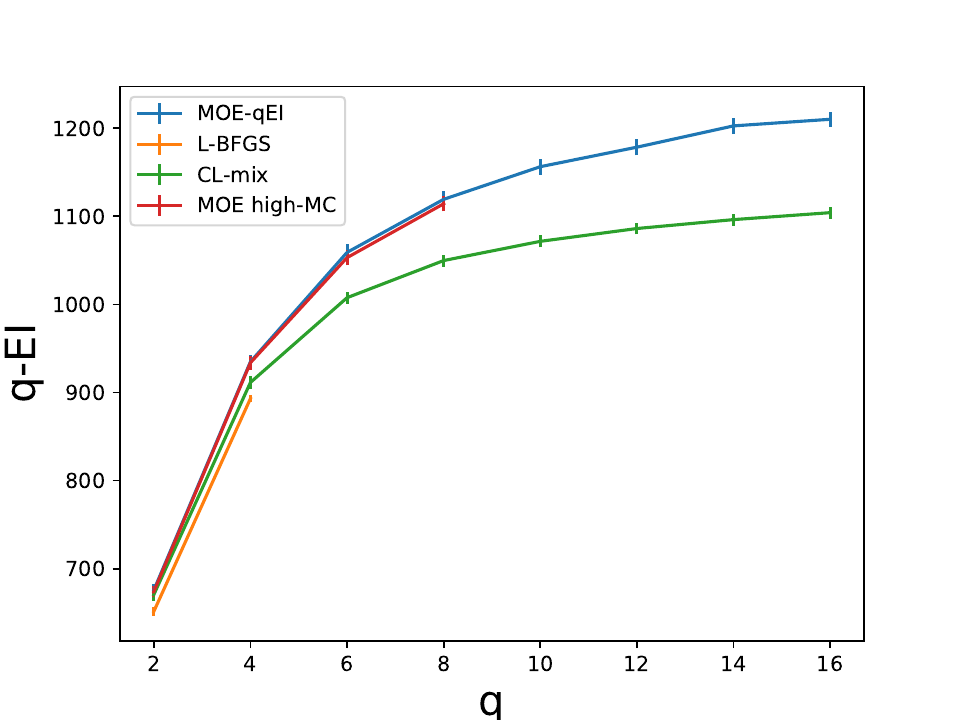}
    }
    \subfloat[1][Hartmann3]{
        \includegraphics[width=0.45\textwidth]{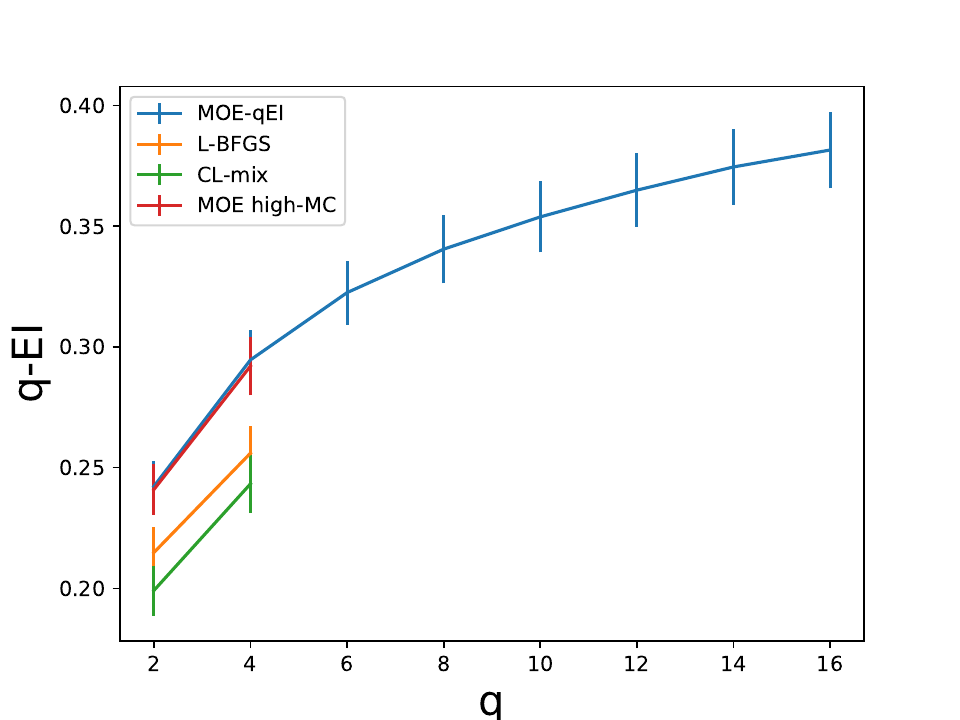}
    }
    \qquad
    \subfloat[2][Ackley5]{
        \includegraphics[width=0.45\textwidth]{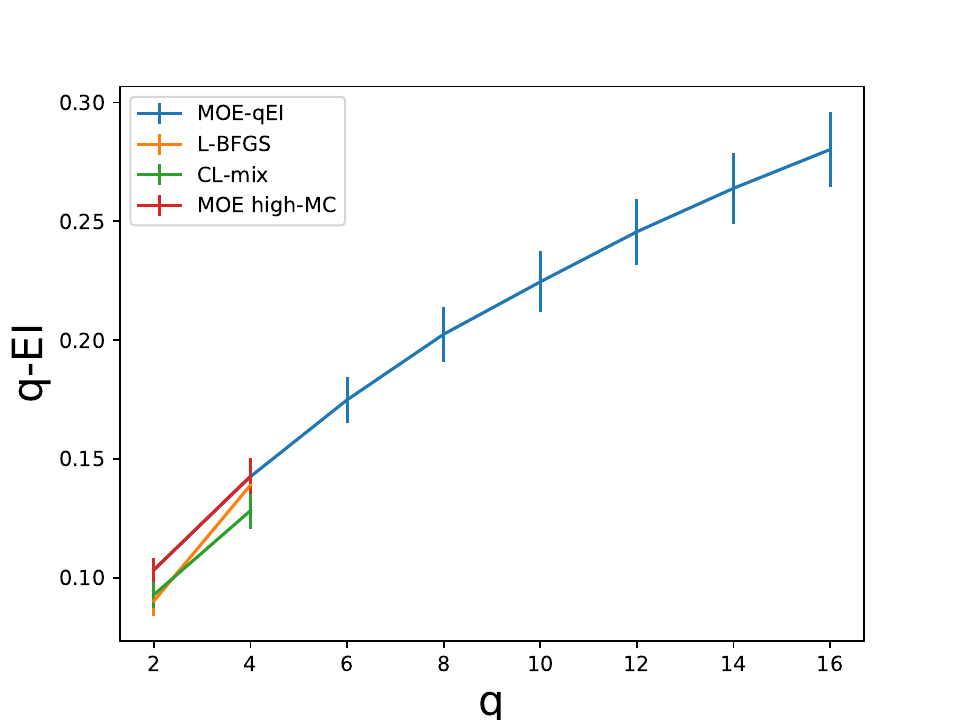}
    }
    \subfloat[3][Hartmann6]{
        \includegraphics[width=0.45\textwidth]{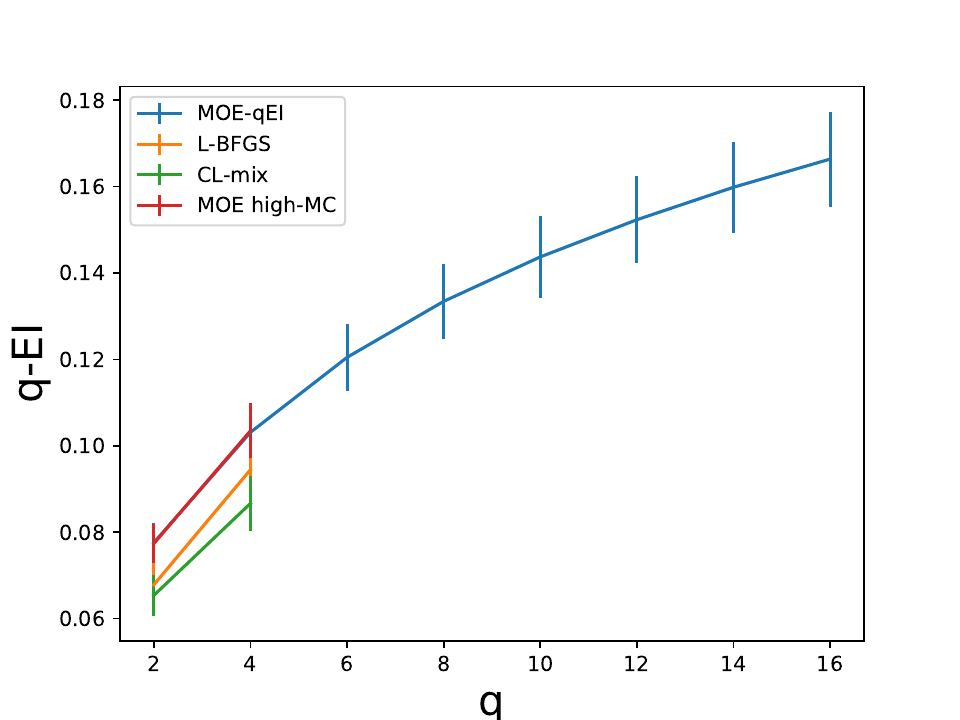}
    }
    \caption{Comparison of algorithms for solving the inner optimization problem: solution quality (maximum $\qEI$) vs. $q$ for different algorithms solving the inner optimization problem. For each test function, we generated 500 instances of the inner optimization problem by randomly sampling $(2d+2)$ points, and the plot shows the average solution quality and 95\% confidence interval for the expected solution quality over 500 problem instances.} 
    \label{fig:qei_opt}
\end{figure}

\begin{figure}[H]
\centering
    \centering
    \subfloat[0][Branin2]{
        \includegraphics[width=0.45\textwidth]{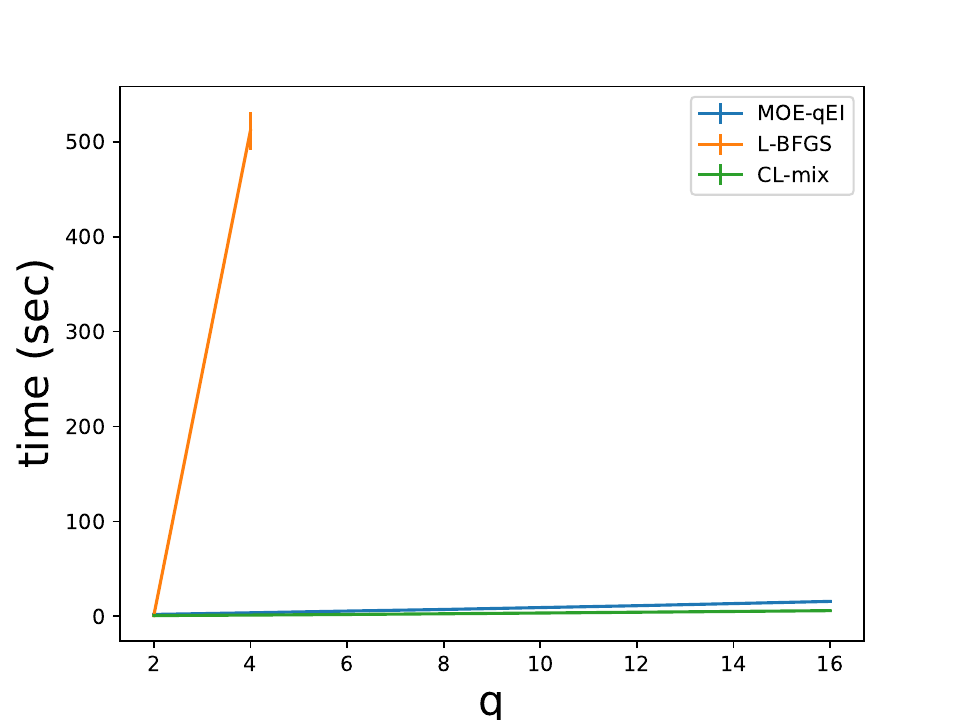}
    }
    \subfloat[1][Hartmann3]{
        \includegraphics[width=0.45\textwidth]{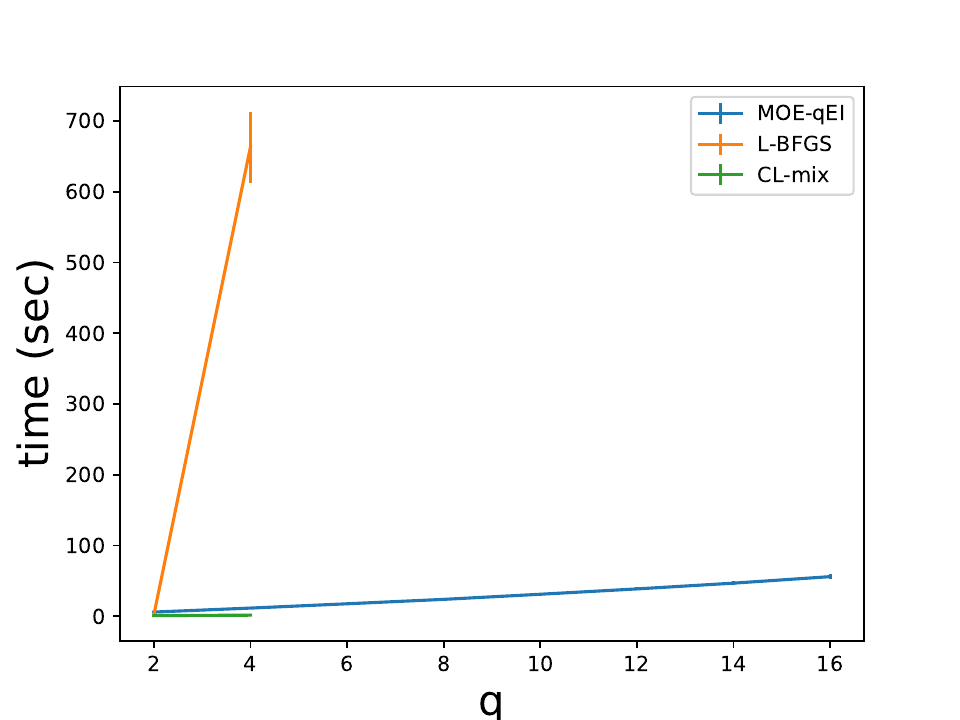}
    }
    \qquad
    \subfloat[2][Ackley5]{
        \includegraphics[width=0.45\textwidth]{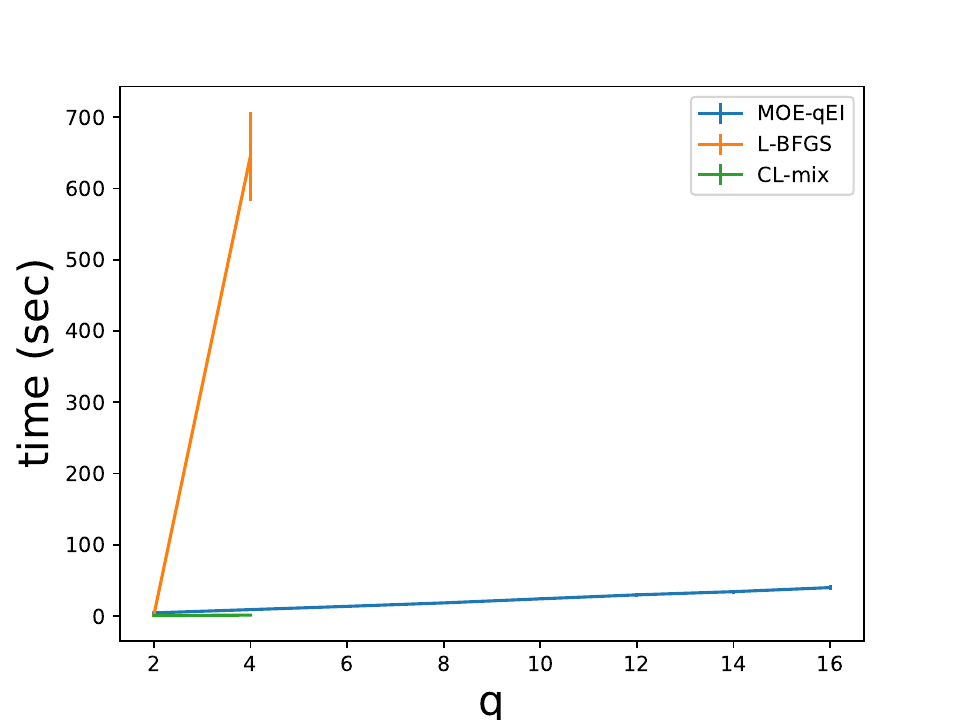}
    }
    \subfloat[3][Hartmann6]{
        \includegraphics[width=0.45\textwidth]{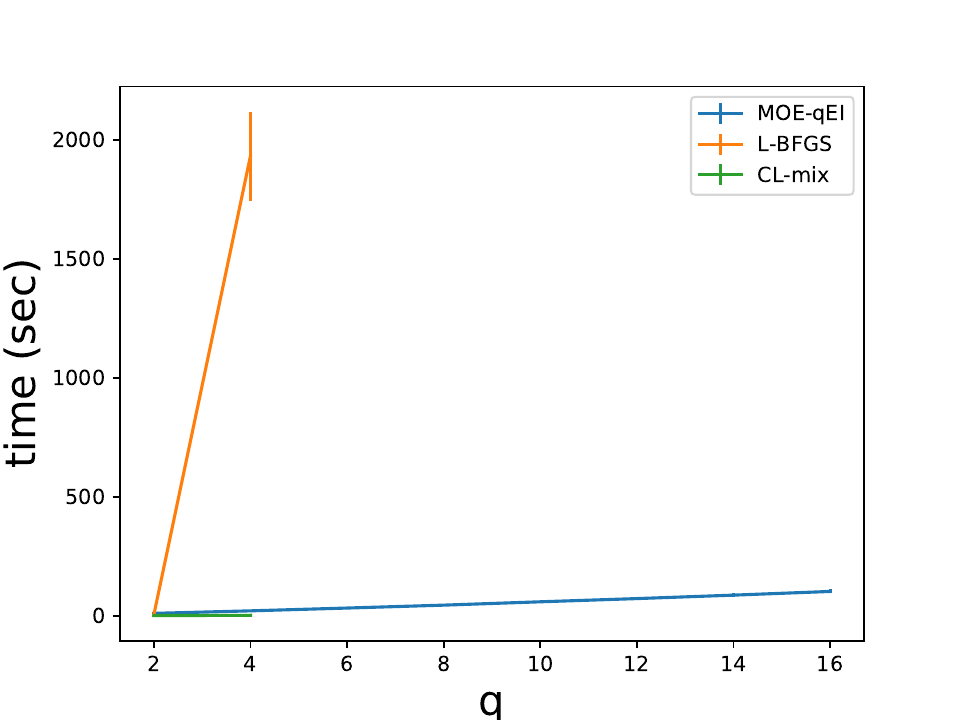}
    }
    \caption{Comparison of algorithms for solving the inner optimization problem: Runtime vs. $q$ for different algorithms solving the inner optimization problem.} 
    \label{fig:qei_time}
\end{figure}

We compare MOE-qEI, CL-mix, and Benchmark 1 in solving the inner optimization problem, in terms of both solution quality and runtime, as shown in Figure~\ref{fig:qei_opt} and \ref{fig:qei_time}. Without surprise, MOE-qEI achieves the best solution quality among the three, and its running time is almost comparable to CL-mix, which is expected to be the fastest approach because it sacrifices solution quality for speed.  We ran Benchmark 1 with $q$ going only up to 4 because its runtime goes up drastically with $q$.
MOE-qEI's runtime scales well as $q$ grows, making it feasible to run in applications with high parallelism. To our surprise, CL-mix achieves competitive solution quality against Benchmark 1, using only a fraction of Benchmark 1's runtime. Therefore, despite the promise of using the closed-form formula for $\qEI$ to fully solve the inner optimization problem, this formula's long runtime and the slow convergence of L-BFGS due to lack of derivative information make Benchmark 1 a less favorable option than CL-mix in practice. 

Figure~\ref{fig:qei_opt} also includes a method called \enquote{MOE high-MC}. This method runs MOE-qEI on GPU with the number of Monte Carlo samples $M$ for the gradient estimator set to $10^7$, much higher than the default setting of 1000. As shown in the figure, the solution quality for \enquote{MOE high-MC} is the same as that of MOE-qEI, which confirms that $M=1000$ is sufficiently large.

\subsection{Comparison on evaluation of $\nabla \qEI$} \label{sec:compare_against_exact_eval}
A recently published book chapter~\cite{marmin2015differentiating}, developed independently and in parallel to this work, proposed a method for computing $\nabla \qEI$ using a closed-form formula derived from~\eqref{eq:fast_ei}, and then proposed to use this formula inside a gradient-based optimization routine to solve~\eqref{eq:maxEI}.
The formula is complex and therefore we do not reproduce it here.  

This formula faces even more severe computational challenges than \eqref{eq:fast_ei};
indeed, it requires $O(q^4)$ calls to multivariate normal CDFs with dimension between $(q-3)$ and $q$.  Because computing high-dimensional multivariate normal CDFs is itself challenging, 
this closed-form evaluation becomes extremely time-consuming.

MOE-qEI's Monte-Carlo based approach to evaluating $\nabla \qEI$ offers three advantages over using the closed-form formula:
first, numerical experiments below suggest that computation scales better with $q$;
second, it can be easily parallelized, with significant speedups possible through parallel computing on graphical processing units (GPUs), as is implemented within the MOE library;
third, by using a small number of replications to make each iteration run quickly, and by using it within a stochastic gradient ascent algorithm that averages noisy gradient information intelligently across iterations, we may more intelligently allocate effort across iterations, only spending substantial effort to estimate gradients accurately late in the process of finding a local maximum.


We first show that computation of exact gradients using our gradient estimator with many replications on a GPU scales better with $q$ through numerical experiments. We compare with closed-form gradient evaluation on a CPU as implemented in the \enquote{DiceOptim} package \citep{diceoptim} and call it \enquote{Benchmark 2}. 
We computed $\nabla \qEI$ at 200 randomly chosen points
from a 2-dimensional design space to obtain a $95\%$ confidence interval for the average computation time.
To make the gradient evaluation in MOE-qEI close to exact,
we increased the number of Monte Carlo samples 
used in the gradient estimator to $10^7$, which ensures that the variance of each component of the gradient is on the order of $10^{-10}$ or smaller
for all $q$ we consider in our experiments. Given the large number of Monte Carlo samples, we use the GPU option in the MOE package to speed up computation.  This GPU implementation is made possible by the trivial parallelism supported by our Monte-Carlo-based gradient estimator, while a massively parallel GPU implementation of closed-form gradient evaluation would be more challenging.

\begin{figure}[H]
\centering
\includegraphics[width=0.5\textwidth]{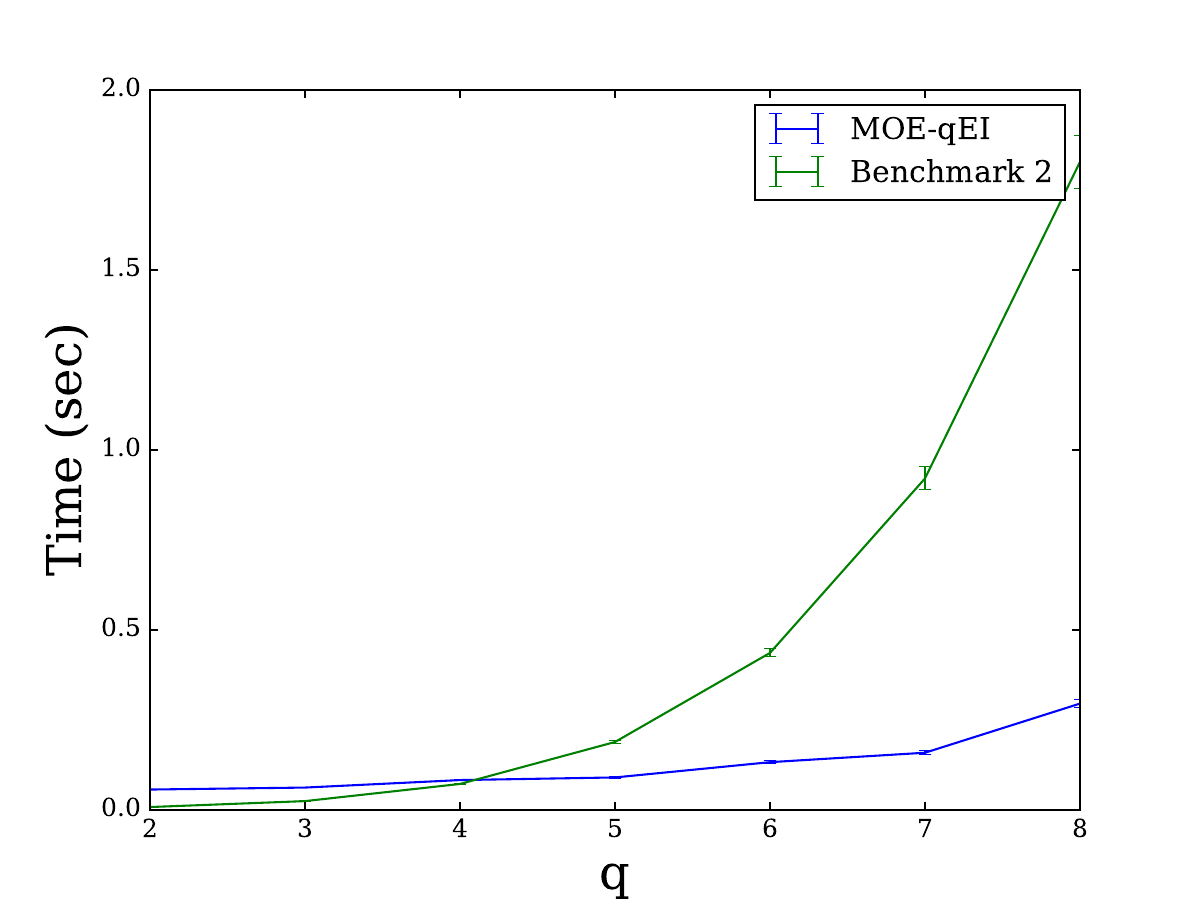}
\caption{Comparison against closed-form evaluation of $\nabla \qEI$: average time to compute $\nabla \qEI$ with high precision v.s. $q$, comparing the gradient-based estimator from MOE-qEI using a large number of samples ($10^7$) in a parallel GPU implementation with the closed-form formula from \cite{marmin2015differentiating}. The stochastic gradient estimator in MOE-qEI scales better in $q$ and is faster when $q \geq 4$.}
\label{fig:time_grad_ei}
\end{figure}

Figure~\ref{fig:time_grad_ei} shows
that computational time for Benchmark 2 increases quickly as $q$ grows, but increases slowly for MOE-qEI's Monte Carlo estimator, with this Monte Carlo estimator being faster when $q \geq 4$.   This difference in performance arises because gradient estimation in MOE-qEI focuses Monte Carlo effort on calculating a single high-dimensional integral, while the closed-form formula decomposes this high-dimensional integral of interest into a collection of other high-dimensional integrals that are almost equally difficult to compute, and the size of this collection grows with $q$.

We may reduce the number of Monte Carlo samples used by MOE-qEI substantially while still providing high-accuracy estimates: if we reduce $M$ from $10^7$ to $10^4$, the variance of each component of the gradient remains below $10^{-7}$.
Since the GPU implementation provides a roughly 100x to 1000x speedup, the CPU-only implementation of our stochastic gradient estimator with this reduced value of $M$ has run-time comparable to or better than the GPU-based results pictured in Figure~\ref{fig:time_grad_ei}, showing that even without the hardware advantage offered by a GPU our stochastic gradient estimator provides high-accuracy estimates faster than \cite{marmin2015differentiating} for $q\ge4$.



Moreover, because stochastic gradient ascent is tolerant to noisy gradients, we may obtain additional speed improvements by reducing the number of Monte Carlo samples even further.  
Using fewer Monte Carlo samples in each iteration has the potential to increase efficiency by only putting effort toward estimating the gradient precisely when we are close to the stationary point, which stochastic gradient ascent performs automatically through its decreasing stepsize sequence.  
Thus, our stochastic gradient estimator is both faster when using a large number of samples to produce essentially exact estimates, and it offers more flexibility in its ability to produce inexact estimates at low computational cost.





\section{Conclusions}
We proposed an efficient method based on stochastic approximation for implementing a conceptual parallel Bayesian global optimization algorithm proposed by \cite{GiLeCa08}. To accomplish this, we used infinitessimal perturbation analysis (IPA) to construct a stochastic gradient estimator and showed that this estimator is unbiased. We also provided convergence analysis of the stochastic gradient
ascent algorithm with the constructed gradient estimator. Through numerical experiments, we demonstrate that our method outperforms the existing state-of-the-art approximation methods.

\bibliographystyle{apalike}
\bibliography{MOE}

\section*{Appendix: proofs of results in the main paper}
\begin{lemma}
\label{lemma:1}
Suppose functions $\lambda_i(x) : \mathbb{R} \mapsto \mathbb{R}$, $i=1,\ldots,m$ are continuously differentiable on a compact interval $\mathcal{X}$. 
Let $\Lambda(x) = \max_{i=1}^m \lambda_i(x)$, and $\overset{\sim}{\mathcal{X}}$ be the set of points where $\Lambda(x)$ fails to be differentiable. 
Then $\overset{\sim}{\mathcal{X}}$ is countable.
\end{lemma}

\begin{proof}{Proof of Lemma~\ref{lemma:1}.} We first consider $m=2$, and later extend to $m>2$. Let $x_0$ be a point of non-differentiability of $\Lambda(x)$. 
Then 
$\lambda_1(x_0) = \lambda_2(x_0)$ and 
$\lambda_1'(x_0) \neq \lambda_2'(x_0)$.
(If the first condition were not true, and suppose $\lambda_1(x_0) > \lambda_2(x_0)$, then continuity of $\lambda_1$ and $\lambda_2$ would imply $\Lambda(x) = \lambda_1(x)$ in an open neighborhood of $x_0$.
If the first condition were true but not the second,
then $\Lambda'(x_0) = \lambda_1'(x_0) = \lambda_2'(x_0)$.)



By continuity of $\lambda_1'$ and $\lambda_2'$, $\exists \delta > 0$ such that $\lambda_1'(x) > \lambda_2'(x)$ for all
$x\in(x_0-\delta, x_0+\delta)$. 
Therefore $\Lambda(x)=\lambda_1(x)$ at $x \in (x_0, x_0+\delta)$
and $\Lambda(x)=\lambda_2(x)$ at $x \in (x_0-\delta, x_0)$.
Thus $\Lambda$ is differentiable on $(x_0-\delta, x_0+\delta) \backslash \{x_0\}$. 

Let $n(x_0)$ be the smallest integer $n\ge1$ such that 
$\Lambda$ is differentiable on $(x_0-1/n, x_0+1/n) \backslash \{x_0\}$ and  let $D(n)$ be the set of non-differentiable points $x$ such that $n(x)=n$. In an interval of length $L$, there can be at most 
$Ln+1$ points in $D(n)$. Hence the set of all non-differentiable points $\overset{\sim}{\mathcal{X}} = \cup_{n=1}^{\infty} D(n)$ is countable.

Now let $m>2$. We show that all points of discontinuity of $\Lambda(x)$ are also points of discontinuity of $\max(\lambda_i(x), \lambda_j(x))$ for at least one pair of $i,j$.
Let $S(x) = \arg\max_i \lambda_i(x)$. 
Using Taylor's theorem, 
for $\Delta \in \mathbb{R}$,
\begin{equation*}
\lambda_i(x+\Delta) = \lambda_i(x) + \lambda_i'(x) \Delta + h_i(x+\Delta) \Delta,
\end{equation*}
where $h_i(\cdot)$ is a function such that $\lim_{\Delta \rightarrow 0} h_i(x+\Delta)=0$. We write the left and right derivative of $\Lambda$ at $x$ as
\begin{equation*}
\begin{split}
\lim_{\Delta \rightarrow 0^+} \frac{\Lambda(x+\Delta) - \Lambda(x)}{\Delta} &= \lim_{\Delta \rightarrow 0^+} \frac{\max \{ \lambda_i(x) + \lambda_i'(x) \Delta + h_i(x+\Delta) \Delta\} - \lambda_{i^*}(x)}{\Delta},\\
&= \lim_{\Delta \rightarrow 0^+} \max \{ \frac{\lambda_i(x) - \lambda_{i^*}(x)}{\Delta} + \lambda_i'(x) + h_i(x)\}, \\
&= \max \{\lambda_i'(x): i \in S \},
\end{split}
\end{equation*}
and
\begin{equation*}
\begin{split}
\lim_{\Delta \rightarrow 0^+} \frac{\Lambda(x) - \Lambda(x-\Delta)}{\Delta} &= \lim_{\Delta \rightarrow 0^+} \frac{\lambda_{i^*}(x) - \max \{ \lambda_i(x) - \lambda_i'(x) \Delta - h_i(x-\Delta) \Delta\}}{\Delta},\\
&= \lim_{\Delta \rightarrow 0^+} \min \{ \frac{\lambda_{i^*}(x) - \lambda_{i}(x)}{\Delta} + \lambda_i'(x) + h_i(x)\}, \\
&= \min \{\lambda_i'(x): i \in S \}.
\end{split}
\end{equation*}
If the left and right derivative are equal, $\Lambda$ is differentiable at $x$.
If not, let $i^+ \in \argmax\{ \lambda_i'(x) : i \in S\}$ and $i^- \in \argmin\{ \lambda_i'(x) : i \in S\}$.
Then $\max(\lambda_{i^+}(x),\lambda_{i^-}(x))$ fails to be differentiable at $x$.

Thus the non-differentiable points of $\Lambda'(x)$ are a subset of the union of the non-differentiable points of $\max(\lambda_i(x), \lambda_j(x))'$ over all $i,j$, and so it is a subset of a finite union of countable sets, which is countable.
\end{proof}

\begin{lemma}
\label{lemma:2}
If $\bm{m}(\bm{X})$ and $\bm{C}(\bm{X})$ are 
differentiable in a neighborhood of $\bm{X}$, and there are no 
duplicated rows in $\bm{C}(\bm{X})$, 
then $P \left(\bm{e}_i [\bm{m}(\bm{X}) + \bm{C}(\bm{X}) \bm{Z}] = \bm{e}_j [\bm{m}(\bm{X}) + \bm{C}(\bm{X}) \bm{Z}]\right)=0$ for any $i \ne j$,
and $\nabla h(\bm{X}, \bm{Z})$ exists almost surely
for any $\bm{X}$.
\end{lemma}

\begin{proof}{Proof of Lemma~\ref{lemma:2}.} Observe that
$h(\bm{X}, \bm{Z}) = \bm{e}_{I^*} \left[ \bm{m}(\bm{X}) + \bm{C}(\bm{X}) \bm{Z} \right]$,
where $I^* \in \argmax_{i=0,\ldots,q} \bm{e}_i \left[ \bm{m}(\bm{X}) + \bm{C}(\bm{X}) \bm{Z} \right] := \mathcal{S}$. 
$\nabla h(\bm{X}, \bm{Z})$ can fail to exist only if $\exists I_1, I_2 \in \mathcal{S}$ with 
$\bm{e}_{I_1} (\frac{\partial \bm{m}(\bm{X})}{\partial x_{ik}} + 
\frac{\partial \bm{C}(\bm{X})}{\partial x_{ik}} \bm{Z}) \neq \bm{e}_{I_2} (\frac{\partial \bm{m}(\bm{X})}{\partial x_{ik}} + 
\frac{\partial \bm{C}(\bm{X})}{\partial x_{ik}} \bm{Z})$. Thus,
\begin{align*}
P(\nabla h(\bm{X}, \bm{Z}) \text{\, does not exist}) &\leq P ( \vert\mathcal{S}\vert \geq 2), \\
&\leq \frac{1}{2} \sum_{i \neq j} P \left(\bm{e}_i [\bm{m}(\bm{X}) + \bm{C}(\bm{X}) \bm{Z}] = \bm{e}_j [\bm{m}(\bm{X}) + \bm{C}(\bm{X}) \bm{Z}]\right),\\
&= \frac{1}{2}\sum_{i \neq j} P \left( \left(\bm{C}(\bm{X})_{i\cdot} - \bm{C}(\bm{X})_{j\cdot} \right) \bm{Z} = m(\bm{X})_j - m(\bm{X})_i \right),
\end{align*}
where $\bm{C}_{i\cdot}(\bm{X})$ is the $i$th row of $\bm{C}(\bm{X})$.
Since $\bm{C}_{i\cdot}(\bm{X}) \neq \bm{C}_{j\cdot}(\bm{X})$, $\{ \bm{Z}: \left(\bm{C}_{i\cdot}(\bm{X}) - \bm{C}_{j\cdot}(\bm{X}) \right) \bm{Z} = m_j(\bm{X}) - m_i(\bm{X}) \}$ 
is subspace of $\mathbb{R}^q$ with dimension smaller than $q$, and
\begin{equation*}
P \left( \left(\bm{C}_{i\cdot}(\bm{X}) - \bm{C}_{j\cdot}(\bm{X}) \right) \bm{Z} = m_j(\bm{X}) - m_i(\bm{X}) \right) = 0 \text{\quad} \forall i \neq j.
\end{equation*}
Hence
$P (\nabla h(\bm{X}, \bm{Z}) \text{\, does not exist}) = 0$.
\end{proof}

\begin{proof}{Proof of Theorem~\ref{thm_grad}.} 
Without loss of generality, we consider the partial derivative with respect to the $k$th component of the $m$th point in $\bm{X}$, that is, $\frac{\partial h(\bm{X}, \bm{Z})}{\partial X_{mk}}$. We use the following result in Theorem 1.2. from~\cite{Glasserman1991}, restated here for convenience:

Suppose the following conditions (1), (2), (3) and (4) hold on a compact interval $\Theta$, then $\E\left[ \xi'(\bm{W}(\theta)\right] = \ell'(\theta)$ on $\Theta$, where $\ell(\theta) = \E\left[\xi(\bm{W}(\theta))\right]$.
\begin{enumerate}
\item[(1)]
For all $\theta \in \Theta$ and $i=1,\ldots,n$, $W_i$ is a.s. differentiable at $\theta$.
\item[(2)]
Define $D_{\xi}$ to be the subset of $\mathbb{R}^n$ on which $\xi$ is continuously differentiable. For all $\theta \in \Theta$, $P\left(\bm{W}(\theta) \in D_{\xi}\right)=1$.
\item[(3)]
$\xi\left(\bm{W}(\cdot)\right)$ is a.s. continuous and piecewise differentiable throughout $\Theta$.
\item[(4)]
$\overset{\sim}{D}$ is countable and 
$\E\left[ \sup_{\theta \notin \overset{\sim}{D}} |\xi'(\bm{W}(\theta)|\right] < \infty$, where $\overset{\sim}{D}$ is the random collection of points in $\Theta$ at which $\xi(\bm{W}(\cdot))$ fails to be differentiable. 
\end{enumerate}

We apply this result with $\theta=X_{mk}$,
$\bm{W}(\cdot)$ equal to the random function mapping $X_{mk}$ to the random vector $\bm{m}(\bm{X}) + \bm{C}(\bm{X})\bm{Z}$,
$\xi(w) = \max_{i=0,1,\ldots,q} w_i$,
and $\overset{\sim}{D}$ equal to the set of $X_{mk}$ at which $h'(\bm{X}, \bm{Z})$ does not exist.

Condition (1) is satisfied because $\bm{m}(\cdot)$ and $\bm{C}(\cdot)$ are assumed differentiable.

For condition (2), the set of points $D_\xi$
at which $\xi$ is continuously differentiable is 
$D_\xi = \{ w \in \mathbb{R}^{q+1} : | \argmax_{i=0,1,\ldots,q} w_i| = 1\}$. Lemma~\ref{lemma:2} implies that the probability of equality between two components of $W(\theta)$ is 0, and so $P(W(\theta) \in D_\xi) = 0$.

For condition (3), it is obvious that $\xi\left(\bm{W}(\cdot)\right)$ is a.s. continuous. Lemma~\ref{lemma:1} implies that the set of non-differentiable points is countable, and therefore $\xi\left(\bm{W}(\cdot)\right)$ is a.s. piecewise differentiable.

For condition (4), first $\overset{\sim}{D}$ is countable by Lemma~\ref{lemma:1}. 
We now show the second part of condition (4).
Fix $\bm{X}$ except for $X_{mk}$. Since the interval is compact and $\bm{m}(\bm{X})$ and $\bm{C}(\bm{X})$ 
are continuously differentiable,
\begin{equation*}
\begin{split}
\sup_{X_{mk}} \left\lvert \frac{\partial\bm{m}_i(\bm{X})}{\partial X_{mk}} \right\rvert &= m_i^* < \infty, \\
\sup_{X_{mk}} \left\lvert \frac{\partial\bm{C}(\bm{X})}{\partial X_{mk}} \right\rvert &= C_{ik}^* < \infty.
\end{split}
\end{equation*}
Then
\begin{equation*}
\E\left[\sup_{X_{mk} \notin \overset{\sim}{D}} \lvert h'(\bm{X}, \bm{Z}) \rvert \right] \leq m^{**} + q C^{**} \E[\lvert Z \rvert]
= m^{**} + \sqrt{\frac{2}{\pi}} q C^{**} < \infty,
\end{equation*}
where $m^{**}=\max_i m^*_i$ and $C^{**}=\max_{i,j} C^*_{ij}$. Therefore, condition (4) is satisfied.

Thus the conditions of Theorem 1.2 from~\cite{Glasserman1991} are satisfied and 
$\nabla \mathbb{E} h(\bm{X},\bm{Z}) = \mathbb{E} \nabla h(\bm{X},\bm{Z})$.
\end{proof}

\begin{proof}{Proof of Theorem~\ref{thm:sga}.} We use a convergence analysis result from Section 5, Theorem~2.3 of \cite{kushner2003stochastic} to 
  prove our theorem, which we first state using 
  our notation and setting: the sequence $\{\bm{X}_n\}$ produced by algorithm \eqref{eq:sga}
  converges to a stationary point almost surely if the following assumptions hold, 
  \begin{enumerate}
    \item 
      $\epsilon_n \rightarrow 0$ for $n \geq 0$ and $\epsilon_n=0$ for $n<0$; 
      $\sum_{n=1}^{\infty} \epsilon_n = \infty$
    \item
      $\sup_n \E \left\vert\bm{G}(\bm{X}_n)\right\vert^2 < \infty$
    \item
      There are functions $\lambda_n(\cdot)$ of $\bm{X}$, which are continuous uniformly 
      in $n$, a continuous function $\overline{\lambda}(\cdot)$ and random variables 
      $\beta_n$ such that
      \begin{equation*}
        \E_n \bm{G}(\bm{X}_n) = \lambda_n(\bm{X}_n) + \beta_n,
      \end{equation*}
      and for each $\bm{X} \in H$,
      \begin{equation*}
        \lim_n \left\vert\sum_{i=n}^{m(t_n+t)} \epsilon_i [\lambda_i(\bm{X}) - \overline{\lambda}(\bm{X})] \right\vert = 0
      \end{equation*}
      for each $t>0$, and $\beta_n \rightarrow 0$ with probability one. The function $m(t_n + \cdot)$ is defined in~\cite[Section 5.1]{kushner2003stochastic}.
    \item
      $\sum_i \epsilon_i^2 < \infty$.
    \item
      There is a continuously differentiable real-valued function $\phi(\cdot)$ such that $\bar{\lambda}(\cdot) = - \nabla \phi(\cdot)$
      and $\phi(\cdot)$ is constant on each connected subset $S_i$ of the set of stationary points.
    \item 
      $a_i(\cdot), i=1,\ldots,p$ are continuously differentiable.
  \end{enumerate}
  
  \cite{kushner2003stochastic} shows that if these conditions are satisfied, then $\{X_n\}$ converges to a unique $S_i$.
  Now we prove that the 6 conditions stated above are indeed satisfied if the assumptions in Theorem~\ref{thm:sga} hold, where $\phi(\bm{X}) = -\qEI(\bm{X})$.
  \begin{enumerate}
    \item 
      Condition 1 is satisfied by assumption 2 in Theorem~\ref{thm:sga}.
      Construction of this sequence has been discussed in Section~\ref{sec:optimization_qEI}.
    \item
      First we assume $M=1$ and treat $M>1$ below. Then $\bm{G}(\bm{X}_n) = \bm{g}(\bm{X}_n, \bm{Z})$, and
      \begin{equation*}
        \begin{split}
          \E \vert\bm{G}(\bm{X}_n) \vert^2 &= \E \sum_{m=1}^q \sum_{k=1}^d \bm{e}_{m,k} \bm{G}(\bm{X}_n)^2, \\
          &= \sum_{m=1}^q \sum_{k=1}^d \E \left( \frac{\partial h(\bm{X}, \bm{Z})}{\partial X_{mk}} \bigg\vert_{\bm{X} = \bm{X}_n} \right)^2, \\
          &= \sum_{m=1}^q \sum_{k=1}^d \E \left[ \bm{e}_{I^*_{\bm{Z}}} \left( \frac{\partial \bm{m}(\bm{X})}{\partial X_{mk}} + \frac{\partial \bm{C}(\bm{X})}{\partial X_{mk}} \bm{Z} \right) \bigg\vert_{\bm{X} = \bm{X}_n} \right]^2, \\
          &\leq \sum_{m=1}^q \sum_{k=1}^d \E \sum_{i=0}^q \left[ \bm{e}_i \left( \frac{\partial \bm{m}(\bm{X})}{\partial X_{mk}} + \frac{\partial \bm{C}(\bm{X})}{\partial X_{mk}} \bm{Z} \right) \bigg\vert_{\bm{X} = \bm{X}_n} \right]^2, \\
          & = \sum_{m=1}^q \sum_{k=1}^d \sum_{i=0}^q \E \left[ \bm{e}_i \left( \frac{\partial \bm{m}(\bm{X})}{\partial X_{mk}} + \frac{\partial \bm{C}(\bm{X})}{\partial X_{mk}} \bm{Z} \right) \bigg\vert_{\bm{X} = \bm{X}_n} \right]^2,
        \end{split}
        \label{}
      \end{equation*}
      where $I^*_{\bm{Z}} = \underset{i=0, \ldots, q}{\arg\max} \,\bm{e}_i \left( \bm{m}(\bm{X_n}) + \bm{C}(\bm{X}_n)\bm{Z} \right)$. 
      Since $\bm{m}(\bm{X})$ and $\bm{C}(\bm{X})$ are continuously differentiable for 
      $\forall \bm{X} \in H$ and $H$ is compact, 
      $\sup_{\bm{X}_n} \left\vert\left\vert \frac{\partial \bm{m}(\bm{X})}{\partial X_{mk}} \right\vert\right\vert_{\infty} < \infty$ 
      and $\sup_{\bm{X}_n} \left\vert\left\vert \frac{\partial \bm{C}(\bm{X})}{\partial X_{mk}} \right\vert\right\vert_{\infty} < \infty$. 
      Thus $\sup_{\bm{X}_n} \E \left[ \left(\bm{e}_i \left( \frac{\partial \bm{m}(\bm{X})}{\partial X_{mk}} + \frac{\partial \bm{C}(\bm{X})}{\partial X_{mk}} \bm{Z} \right) \right)^2 \right] < \infty$, 
      and we can conclude that $\sup_n \E\vert\bm{G}(\bm{X}_n)\vert^2 < \infty$. 
    
      If $M > 1$, $\bm{G}(\bm{X}_n)$ is an average of i.i.d. samples of $g(X_n,Z)$. Then $\E \vert g(X_n,Z) \vert^2
      = \frac{1}{M} \E \vert g(X_n,Z) \vert^2$. We have just showed that $\sup_n \E \vert\bm{G}^1(\bm{X}_n) \vert^2$ is finite, and thus $\sup_n \E \vert\bm{G}(\bm{X}_n) \vert^2$ is finite.
      Therefore, condition 2 is satisfied.
    \item
      Define a function $\bar{\bm{g}}(\cdot)$ on $H$ by $\bar{\bm{g}}(X)=\E \bm{g}(\bm{X},\bm{Z})$.  
      Then $\E_n \bm{G}(\bm{X}_n) = \E \bm{g}(\bm{X}_n, \bm{Z})$. 
      Then, since our assumptions meet the requirements for Theorem~\ref{thm_grad} ($\bm{\Sigma}(\bm{X})$ being positive definite implies that $\bm{C}(\bm{X})$ has no duplicate rows),  we know $\bar{\bm{g}}(\bm{X})=\nabla \E h(\bm{X},\bm{Z})$.
      We will show $\nabla \E h(\bm{X}, \bm{Z})$ is continuous on $H$.
      Letting $\lambda_n(\cdot) \equiv \overline{\lambda}(\cdot) \equiv \overline{\bm{g}}(\cdot)$, and $\beta_n = 0$, the first half of condition 3 will then be satisfied.
      Since $\lambda_n(\cdot) \equiv \overline{\lambda}(\cdot)$, the second half of condition 3 is satisfied from the fact that the summand is 0.      
      
      We now show $\nabla \mathbb{E} h(\bm{X},\bm(Z)$ is continuous.
      First, we let $\bm{m}'(\bm{X}) = f_n^* - \bm{\mu}(\bm{X})$, and $\bm{C}'(\bm{X}) = -\bm{L}(\bm{X})$, which are the first through the $q$th entries and rows of $\bm{m}(\bm{X})$ and $\bm{C}(\bm{X})$ respectively. Note that $\bm{\Sigma}(\bm{X}) = \bm{C}'(\bm{X}) \bm{C}'^T(\bm{X})$. Then
      \begin{equation} \label{eq:ec_eh}
           \E \left[ h(\bm{X}, \bm{Z}) \right] = \E \left[ \E \left[ h(\bm{X}, \bm{Z}) \mid t_{-i} \right] \right],
      \end{equation}
      where $t_{-i} = \{\bm{e}_{\ell} \left( \bm{m}'(\bm{X}) + \bm{C}'(\bm{X}) \bm{Z} \right), \forall \ell \ne i, \ell = 1, \ldots, q\}$ for some $i=1,\ldots,q$. 
      
      Fix $i$, and letting $t_i = \bm{e}_i \left( \bm{m}'(\bm{X}) + \bm{C}'(\bm{X}) \bm{Z} \right)$, we know that $t_i$ given $t_{-i}$ has a normal
      distribution: 
      \begin{equation*}
          t_i \mid t_{-i} \sim \mathcal{N}\left(\mu(\bm{X}, t_{-i}), \sigma^2(\bm{X})\right),
      \end{equation*}
      where
      \begin{equation*}
       \begin{split}
      \mu(\bm{X}, t_{-i}) &= m'_i(\bm{X}) + \Sigma_{i, -i}(\bm{X}) \Sigma_{-i, -i}^{-1}(\bm{X}) \left(t_{-i} - m'_{-i}(\bm{X}) \right) \\
      &= a(\bm{X}) + \sum_{\ell \ne i} b_{\ell}(\bm{X}) t_{\ell},\\
      \sigma^2(\bm{X}) &= \Sigma_{i,i}(\bm{X}) - \Sigma_{i,-i}(\bm{X}) \Sigma_{-i,-i}^{-1}(\bm{X}) \Sigma_{-i,i}(\bm{X}).
       \end{split}
      \end{equation*}
       Note that $\sigma^2(\bm{X})$ is the Schur complement of $\bm{\Sigma}(\bm{X})$, and
       since $\bm{\Sigma}(\bm{X})$ is positive definite, we know that both $\bm{\Sigma}_{-i,-i}(\bm{X})$ and $\sigma^2(\bm{X})$ are positive definite. Knowing the distribution of $t_i$ given $t_{-i}$, we can write the inner expectation of~\eqref{eq:ec_eh} as
      \begin{equation} \label{eq:ec_f}
      \begin{split}
       f(\bm{X}, t_{-i}) &= \E \left[ h(\bm{X}, \bm{Z}) \mid t_{-i} \right],\\
       &= \left(\mu(\bm{X}, t_{-i}) - t_{-i}^*\right) \Phi\left( \frac{\mu(\bm{X}, t_{-i})-t_{-i}^*}{\sigma(\bm{X})}\right) + \sigma(\bm{X}) \phi \left(\frac{\mu(\bm{X}, t_{-i})-t_{-i}^*}{\sigma(\bm{X})}\right) + t_{-i}^*,
      \end{split}
      \end{equation}
      where $t_{-i}^* = \max(t_{-i}, 0)$. Without loss of generality, we only look at $j$th component of the gradient, and we have
      \begin{equation} \label{eq:ec_df}
      \begin{split}
       \frac{\partial f(\bm{X}, t_{-i})}{\partial \bm{X}_j} &= \frac{\partial \mu(\bm{X}, t_{-i})}{\partial \bm{X}_j} \Phi \left(\frac{\mu(\bm{X}, t_{-i})-t^*_{-i}}{\sigma(\bm{X})} \right) + \frac{\partial \sigma(\bm{X})}{\partial \bm{X}_j} \phi \left(\frac{\mu(\bm{X}, t_{-i})-t^*_{-i}}{\sigma(\bm{X})} \right), \\
       &= \left(\frac{\partial a(\bm{X})}{\partial \bm{X}_j} + \sum_{\ell \ne i} \frac{\partial b_{\ell}(\bm{X})}{\partial \bm{X}_j} t_{\ell} \right) \Phi \left( \frac{\mu(\bm{X}, t_{-i})-t^*_{-i}}{\sigma(\bm{X})} \right) + \frac{\partial \sigma(\bm{X})}{\partial \bm{X}_j} \phi \left( \frac{\mu(\bm{X}, t_{-i})-t^*_{-i}}{\sigma(\bm{X})} \right). \\
       \end{split}
      \end{equation}
      Since $\bm{\Sigma}_{-i,-i}(\bm{X})$ is positive definite, and the matrix inverse is a continuous function when restricted to the set of positive definite matrices, and the composition of two continuous functions is continuous, we have that 
     $\bm{\Sigma}_{-i,-i}^{-1}(\bm{X})$ is continuously differentiable.
     Moreover, $\bm{\mu}(\bm{X})$ and $\bm{\Sigma}(\bm{X})$ are assumed continuously differentiable in the statement of the theorem. This together implies continuous differentiability of $a(\bm{X}), b_{\ell}(\bm{X})$ and $\sigma(\bm{X})$. Then
      \begin{equation*}
       \left\lvert \frac{\partial f(\bm{X}, t_{-i})}{\partial \bm{X}_j} \right\rvert \le a^* + \sum_{\ell \ne i} b^*_{\ell} \lvert t_{\ell} \rvert + \sigma^*
      \end{equation*}
      for all $\bm{X} \in H$ and $t_{-i}$, where
      \begin{equation*}
       \begin{split}
           a^* = \sup_{\bm{X}\in H} \left\lvert \frac{\partial a(\bm{X})}{\partial \bm{X}_j} \right\rvert, \\
           b_{\ell}^* = \sup_{\bm{X}\in H} \left\lvert \frac{\partial b_{\ell}(\bm{X})}{\partial \bm{X}_j} \right\rvert, \\
           \sigma^* = \sup_{\bm{X}\in H} \left\lvert \frac{\partial \sigma(\bm{X})}{\partial \bm{X}_j} \right\rvert,
       \end{split}
      \end{equation*}
      and $a^*, b_{\ell}^*, \sigma^*$ are finite because $H$ is compact.
      
      Since $t_{\ell}$ has a normal distribution, $\E[\lvert t_{\ell} \rvert] < \infty$ and therefore
      \begin{equation*}
       \E\left[ \lvert \frac{\partial f(\bm{X}, t_{-i})}{\partial \bm{X}_j} \rvert \right] <  \infty.
      \end{equation*}
      With the conditions above, we can apply Theorem~1.2 in~\cite{Glasserman1991}, and have
      \begin{equation} \label{eq:ec_grad_interchange}
       \frac{\partial \E \left[ f(\bm{X}, t_{-i}) \right]}{\partial \bm{X}_j} = \E \left[ \frac{\partial f(\bm{X}, t_{-i})}{\partial \bm{X}_j} \right].
      \end{equation}
      Moreover, we can write~\eqref{eq:ec_grad_interchange} as
      \begin{equation*}
       \E \left[ \frac{\partial f(\bm{X}, t_{-i})}{\partial \bm{X}_j} \right] = \E \left[ \frac{\partial f(\bm{X}, t'_{-i})}{\partial \bm{X}_j} \phi(t'_{-i}) \right],
      \end{equation*}
      where each component of $t'_{-i}$ is an independent uniform random variable on $(-\infty, \infty)$, and $\phi(\cdot)$ is the multivariate normal probability density function for $t_{-i}$. Define the function
      \begin{equation*}
       G(\bm{X}, t'_{-i}) = \frac{\partial f(\bm{X}, t'_{-i})}{\partial \bm{X}_j} \phi(t'_{-i}),
      \end{equation*}
      From Lemma~\ref{lemma:bound_div}, $\lvert \frac{\partial G(\bm{X}, t'_{-i})}{\partial \bm{X}_k} \rvert$ is bounded by a finite constant for all $t'_{-i}$. Thus $G(\bm{X}, t'_{-i})$ is Lipschitz continuous in $\bm{X}$ with some constant $K$. 
      
      Given any $\epsilon > 0$, we let $\delta = \epsilon / K$, and for any $\bm{X}'$ such that $\lvert \bm{X}' - \bm{X} \rvert < \delta$, $\lvert G(\bm{X}', t'_{-i}) - G(\bm{X}, t'_{-i}) \rvert < K \cdot \delta = \epsilon$. Hence, $\lvert \E \left[G(\bm{X}', t'_{-i})\right] - \E \left[G(\bm{X}, t'_{-i})\right] \rvert \leq \E \left[ \lvert G(\bm{X}', t'_{-i}) - G(\bm{X}, t'_{-i}) \rvert \right] < \epsilon$ by Jensen's inequality. 
      Therefore, $\E \left[ G(\bm{X}, t'_{-i})\right]
      = \partial \E[h(\bm{X},\bm{Z})] / \partial \bm{X}_j$
      is continuous at any $\bm{X} \in H$.

    \item
      Condition 4 is satisfied by assumption 2 in Theorem \ref{thm:sga}.
    \item
      From the proof of condition 3, we know $\bar{\lambda}(\cdot) = \overline{\bm{g}}(\cdot) = \nabla \qEI(\cdot)$, 
      and thus $\phi(\cdot)=-\qEI(\cdot)$. We have shown that $\bar{\bm{g}}(\cdot)$ 
      is continuous, and it is also trivial to see $\phi(\cdot)$ is constant on each $S_i$.
      Therefore, condition 5 is satisfied.
    \item 
      This is satisfied by assumption 1 in Theorem~\ref{thm:sga}.
  \end{enumerate}
  In conclusion, all conditions are satisfied and therefore $\{\bm{X}_n\}$ converges 
  to a connected set of stationary points almost surely. From Lemma~\ref{lemma:polyak-ruppert}, the Polyak-Ruppert average $\overline{\bm{X}_n(\omega)}$ of the sequence $\{\bm{X}_n(\omega)\}$ converges to the same set as the sequence $\{\bm{X}_n(\omega)\}$ for every $\omega$. 
\end{proof}

\begin{lemma} \label{lemma:bound_div}
If $\bm{\mu}(\bm{X})$, $\bm{\Sigma}(\bm{X})$ are twice differentiable, and $\bm{\Sigma}(\bm{X})$ is positive definite, then $\lvert \frac{\partial G(\bm{X}, t'_{-i})}{\partial \bm{X}_k} \rvert$ is bounded by a finite constant for all $t'_{-i}$.
\end{lemma}

\begin{proof}{Proof of Lemma~\ref{lemma:bound_div}.} We can write $G(\bm{X}, t'_{-i})$ as
\begin{equation*}
 G(\bm{X}, t'_{-i}) = \frac{1}{\sqrt{(2\pi)^{q-1} \lvert \Sigma_{-i,-i}(\bm{X}) \rvert}} \frac{\partial f(\bm{X}, t'_{-i})}{\partial \bm{X}_j} e ^{-\frac{1}{2} \left(t'_{-i} - \bm{m}'_{-i}(\bm{X}) \right)^T \Sigma^{-1}_{-i,-i}(\bm{X}) \left(t'_{-i} - \bm{m}'_{-i}(\bm{X}) \right)}.
\end{equation*}
Since $\bm{\mu}(\bm{X})$, $\bm{\Sigma}(\bm{X})$ are twice differentiable, and $\bm{\Sigma}(\bm{X})$ is positive definite, we can
take the partial derivative with respect to $\bm{X}_k$. With some algebra, we can show
\begin{equation} \label{eq:ec_lemma}
\left\lvert \frac{\partial G(\bm{X}, t'_{-i})}{\partial \bm{X}_k} \right\rvert < \sum_r c_r \lvert P_r(t'_{-i}) \rvert e ^{-\frac{1}{2} \left(t'_{-i} - \bm{m}'_{-i}(\bm{X}) \right)^T \Sigma^{-1}_{-i,-i}(\bm{X}) \left(t'_{-i} - \bm{m}'_{-i}(\bm{X}) \right)},
\end{equation}
where each $P_r(t'_{-i})$ is a monomial in components of $t'_{-i}$ with coefficient 1 and order ranging between 0 and 2, and $0 < c_r < \infty$.
Let $\bm{L}_{-i,-i}(\bm{X})$ be the Cholesky decomposition of $\bm{\Sigma}^{-1}_{-i,-i}(\bm{X})$, and $z_{-i} = \bm{L}_{-i,-i}(\bm{X}) \left( t'_{-i} - m'_{-i}(\bm{X})\right)$.  Invertibility of $\Sigma^{-1}_{-i,-i}(\bm{X})$ implies that $t'_{-i}$ can be written in terms of $z_{-i}$
Substitute $z_{-i}$ into~\eqref{eq:ec_lemma} and, and we get
\begin{equation*}
     \left\lvert \frac{\partial G(\bm{X}, z_{-i})}{\partial \bm{X}_k} \right\rvert < \sum_r c_r' \lvert P_r'(z_{-i}) \rvert e ^{-\frac{1}{2} z_{-i}^T z_{-i}},
\end{equation*}
where each $c_r'$ is a finite constant and each $P_r'(\cdot)$ is a monomial with coefficient 1 and order between 0 and 2. Without loss of generality, we assume the first component, $z_0$, has the largest absolute value among $z_{-i}$. Then
\begin{equation} \label{eq:ec_lemma2}
     \left\lvert \frac{\partial G(\bm{X}, z_{-i})}{\partial \bm{X}_k} \right\rvert < \sum_r c_r' \lvert P_r'(z_0) \rvert e ^{-\frac{1}{2} z_0^2}.
\end{equation}
We can show that $\lvert x^p \rvert e^{-\frac{1}{2}x^2} \leq p^{\frac{p}{2}} e^{-\frac{1}{2}p}, \forall x\in(-\infty,\infty)$. Therefore, each summand in~\eqref{eq:ec_lemma2} is bounded by a constant. 
\end{proof}

\begin{lemma} \label{lemma:polyak-ruppert}
Let $\{\bm{X}_n: n \geq 1\}$ be a sequence in $H$, where $H$ is compact, converging to a set $A$. Let $\overline{\bm{X}}_n = \frac{1}{n} \sum_{m=1}^n \bm{X}_m$. 
Then $\{\overline{\bm{X}}_n : n\geq 1\}$ also converges to $A$.
\end{lemma}
\begin{proof}{Proof of Lemma~\ref{lemma:polyak-ruppert}.} 
Let $\rho_A(\bm{X}) := \inf\{ \lvert\lvert \bm{X} - \bm{X'} \rvert\rvert : \bm{X}' \in A\}$ denote the distance to $A$,
where $\lvert\lvert \cdot \rvert\rvert$ denotes the $\mathrm{L}_2$ norm. $\rho_A(\cdot)$ is convex.

Let $\epsilon >0$.
Since $\bm{X}_n$ converges to $A$, $\exists N_{\epsilon}$ such that  $\rho_A(\bm{X}_n) < \epsilon$ 
for all $n > N_{\epsilon}$.  For $n > N_\epsilon$,
\begin{equation*}
     \overline{\bm{X}}_n = \frac{N_{\epsilon}}{n} \overline{\bm{X}}_{N_{\epsilon}} +
     \left(1-\frac{N_{\epsilon}}{n}\right) \cdot \frac{1}{n-N_{\epsilon}} \sum_{m=N_{\epsilon}+1}^{n} \bm{X}_m.
\end{equation*}

Then
\begin{equation}
\begin{split}
\rho_A(\overline{\bm{X}}_n)
&= 
\rho_A\left(
\frac{N_{\epsilon}}{n} \overline{\bm{X}}_{N_{\epsilon}} +
\left(1-\frac{N_{\epsilon}}{n}\right)
\frac1{n-N_\epsilon}\sum_{m=N_{\epsilon}+1}^{n}\bm{X}_m \right) \\
    &\le \frac{N_{\epsilon}}{n}
    \rho_A(\overline{\bm{X}}_{N_{\epsilon}})  +
    \left(1-\frac{N_{\epsilon}}{n}\right) \frac{1}{n-N_{\epsilon}} \sum_{m=N_{\epsilon}+1}^{n}\rho_A\left(\bm{X}_m \right) \\
    &\leq \frac{N_{\epsilon}}{n} \cdot C + \left(1-\frac{N_{\epsilon}}{n}\right) \cdot \epsilon,
\end{split}
\end{equation}
where $C := \sup_{\bm{X} \in H} \rho_A(\bm{X})$ is finite. Let $\epsilon'>0$, and choose $\epsilon = \frac{\epsilon'}{2}$.
Let $n>N_\epsilon$ be such that $\frac{N_{\epsilon}}{n}C + (1-\frac{N_{\epsilon}}{n}) \epsilon \leq \epsilon' = 2\epsilon$. Then $\forall n' > n$, $\rho_A(\overline{\bm{X}_{n'}}) \leq \epsilon'$. 
\end{proof}

\section*{Appendix: Choice of batch size $q$}\label{sec:choice-of-q}
In the body of the paper we discussed parallel optimization using a given level of parallelism $q$.   Some situations permit a straightforward choice of $q$ but in others potential value can be gained by thinking carefully about this choice.  While a detailed investigation into this choice is beyond the scope of this paper, we provide a high-level discussion here.  We leave a more detailed investigation to future work.

If individual function evaluations cannot be parallelized and computation is performed on previously-purchased computers (rather than in a cloud environment) with a fixed maximum amount of parallelism $q_{\max}$, then $q$ should be set to  $q_{\max}$ to make use of all parallel resources.  This maximizes the quality of the solution that one can find in a given wall-clock time.

If, however, individual function evaluations {\textit can} be parallelized, then we might choose to use more parallel resources per evaluation but run fewer function evaluations at a time.
Let $\hat{q}$ denote the level of parallelism used for a single function evaluation. 
As we increase $\hat{q}$, our function evaluations become faster, but we must reduce $q$.  Thus, the wall-clock time required to find a solution with a given quality may improve or degrade.  Finding the right level of $\hat{q}$ depends on how much  additional parallelism improves runtime of function evaluations versus how much it reduces the number of batches required in optimization.

To formalize this, consider the synchronous setting where $\hat{q}$ is fixed across batches 
and the time per function $T$ evaluation does not depend on the point $x$ evaluated.  
Let $\gamma(\hat{q})$ denote the {\textit speedup} \citep{quinn1994parallel}, i.e., the ratio of time required per function evaluation when executed in serial to the time required when using $\hat{q}$ parallel resources.  Thus, the time required for one function evaluation is $T / \gamma(\hat{q})$.
Initially the speedup is equal to $\hat{q}$ 
($\gamma(1)=1$) but then typically grows more 
slowly than $\hat{q}$ as $\hat{q}$ rises.
We define a notation equivalent to speedup to describe the efficiency of optimization with a given degree of parallelism.
Let $N$ denote the number of function evaluations required by EI with sequential evaluations to find a solution with a given level of quality.
Let $\beta(q)$ denote the ratio of $N$ to the number of batches required for $q$-EI to find a solution with equivalent quality, so that $N/\beta(q)$ is the number of batches required.

Then, the wall-clock time required is $\frac{TN}{\gamma(\hat{q}) \beta(q)}$. 
We seek to find $q$ that minimizes this expression 
subject to the constraint that $q \hat{q} \le q_{\max}$.  Under the mild assumption that both $\gamma$ and $\beta$ are non-decreasing, this will be achieved when $\hat{q} = \lfloor q_{\max} / q \rfloor$.
Optimality will thus be achieved at the $q\ge1$ that maximizes $\gamma(\lfloor q_{\max}/q \rfloor)\beta(q)$.

We may use similar thinking to guide the choice of $q$ in cloud computing environments.  Here, the level of parallelism is not fixed, and instead one may rent more parallelism.  Let $c_1$ be the cost billed by the cloud computing environment per CPU hour. Let $c_2$ be the opportunity cost associated with taking one more hour of wall-clock time to find a solution of the desired quality. 
Then, the optimal $q$ and $\hat{q}$ are found by solving
\begin{equation*}
\min_{q, \hat{q}} \frac{TN}{\gamma(\hat{q}) \beta(q)} \left( c_1 \hat{q} q + c_2 \right).
\end{equation*}

While this brief analysis represents a framework with which one could begin to choose $q$ and $\hat{q}$, substantial work remains.  First,  $\gamma(\hat{q})$ and $\beta(q)$ are generally unknown, and one would need to estimate them from data in a given problem.  Second, there may be substantial benefit to varying $\hat{q}$ and $q$ through the course of optimization. Intuitively, there is little efficiency loss with large $q$ for early batches, where one is exploring broadly across the domain.  In later batches, however, information from recent evaluations may be quite useful when done sequentially, and there may be larger efficiency loss from larger $q$.  Third, one may wish to go beyond synchronous evaluations with deterministic evaluation time constant in $x$ and a deterministic number of evaluations required to reach a desired solution quality.  Fourth, one may wish to optimize the solution quality desired relative to cloud computing and opportunity costs, as in \cite{chick2009economic,chick2012sequential}. Fifth, much would be learned from demonstrating such a framework on a range of practical problems.
We leave such explorations to future work.
\end{document}